\documentclass[a4paper]{article}
    \usepackage[T1]{fontenc}
    \makeatletter
    \def\maxwidth{\ifdim\Gin@nat@width>\linewidth\linewidth
    \else\Gin@nat@width\fi}
    \makeatother
    
    \usepackage{caption}
     \usepackage[section]{placeins}
    \usepackage{adjustbox,xcolor,enumerate,geometry,amsmath,amssymb,textcomp,mathpazo} 

\usepackage{amssymb, amsmath, amsfonts, amsthm, graphicx, subfig, color, lineno,graphicx, prettyref,soul,array,booktabs,lineno}
    \AtBeginDocument{%
        \def\PYZsq{\textquotesingle}
    }
    \usepackage{upquote,eurosym,fancyvrb,grffile,hyperref,longtable,booktabs} 
    \usepackage[mathletters]{ucs} 
    \usepackage[utf8x]{inputenc} 
    \usepackage[inline]{enumitem} 
    \usepackage[normalem]{ulem} 
    \usepackage{hyperref}
    \definecolor{urlcolor}{rgb}{0,.145,.698}
    \definecolor{linkcolor}{rgb}{.71,0.21,0.01}
    \definecolor{citecolor}{rgb}{0,.145,.698}
    \definecolor{ansi-black}{HTML}{3E424D}
    \definecolor{ansi-black-intense}{HTML}{282C36}
    \definecolor{ansi-red}{HTML}{E75C58}
    \definecolor{ansi-red-intense}{HTML}{B22B31}
    \definecolor{ansi-green}{HTML}{00A250}
    \definecolor{ansi-green-intense}{HTML}{007427}
    \definecolor{ansi-yellow}{HTML}{DDB62B}
    \definecolor{ansi-yellow-intense}{HTML}{B27D12}
    \definecolor{ansi-blue}{HTML}{208FFB}
    \definecolor{ansi-blue-intense}{HTML}{0065CA}
    \definecolor{ansi-magenta}{HTML}{D160C4}
    \definecolor{ansi-magenta-intense}{HTML}{A03196}
    \definecolor{ansi-cyan}{HTML}{60C6C8}
    \definecolor{ansi-cyan-intense}{HTML}{258F8F}
    \definecolor{ansi-white}{HTML}{C5C1B4}
    \definecolor{ansi-white-intense}{HTML}{A1A6B2}
    
    \DefineVerbatimEnvironment{Highlighting}{Verbatim}{commandchars=\\\{\}}

    \newtheorem{theorem}{Theorem}[section]

    \newtheorem{corollary}[theorem]{Corollary}
    \newcommand{\R}{\mathbb R}

    \newcommand*{\affaddr}[1]{#1} 
    \newcommand*{\affmark}[1][]{\textsuperscript{#1}}
    \newcommand*{\email}[1]{\texttt{#1}}

\title{The SWAG Algorithm; a Mathematical Approach that Outperforms Traditional Deep Learning. Theory and Implementation}

\author{%
Saeid Safaei\affmark$^{1}$, 
Vahid Safaei\affmark$^{2}$, 
Solmazi Safaei\affmark$^{2}$, 
Zerotti Woods\affmark$^{3}$, \\
Hamid R. Arabnia*\affmark$^{1}$, 
Juan B. Gutierrez*\affmark$^{1,3,4}$\\
\affaddr{
    \affmark$^{[1]}$ Department of Computer Science, University of Georgia
}\\
\affaddr{
    \affmark$^{[2]}$ Mechanical Engineering, Yasouj University
}\\
\affaddr{
    \affmark$^{[3]}$  Department of Mathematics, University of Georgia
}\\
\affaddr{
    \affmark$^{[4]}$  Institute of Bioinformatics, University of Georgia
}\\
\email{\{ssa,zerotti.woods25,hra,jgutierr\}@uga.edu}\\
* Joint corresponding authors. 
}

\begin{document}
\maketitle
\begin{abstract}
The performance of artificial neural networks (ANNs) is influenced by  weight initialization, the nature of activation functions, and their architecture. There is a wide range of activation functions that are traditionally used to train a neural network, e.g. sigmoid, tanh, and Rectified Linear Unit (ReLU). A widespread practice is to use the same type of activation function in all neurons in a given layer. 
In this manuscript, we present a type of neural network in which the activation functions in every layer form a polynomial basis; we name this method SWAG after the initials of the last names of the authors. 
We tested SWAG on three complex highly non-linear functions as well as the MNIST handwriting data set. SWAG outperforms and converges faster than the state of the art performance in fully connected neural networks.
Given the low computational complexity of SWAG, and the fact that it was capable of solving problems current architectures cannot, it has the potential to change the way that we approach deep learning. 
\end{abstract}
\section{Introduction}
Deep learning allows computational models that are composed of multiple processing layers, to learn very abstract representations of data\cite{lecun_deep_2015}. There has been reports of many successes using deep neural networks (DNNs) in areas such as computer vision, speech recognition, language processing, drug discovery, genomics,  and a host of other areas.\cite{janzamin_beating_2015}.
DNNs have allowed us to solve difficult problems and have motivated extensive work to understand their theoretical properties \cite{hayou2018selection}.

The process of how to effectively train a DNN is a complicated task and has been proven to be an NP-Complete problem \cite{blum_training_1989}. Features such as weight initialization, the nature of activation functions, and network architecture can affect the training process of a neural network \cite{schmidhuber2015deep} \cite{hayou2018selection}  \cite{ramachandran2018searching}. In particular,  some choices of activation functions or network architectures can cause loss of information or may increase the amount of time needed to train a DNN \cite{hayou2018selection}\cite{zhang2018effectiveness}\cite{chung2016deep}\cite{lin2017does}.
 
The question of how to effectively find  the best set of nonlinear activation functions is challenging \cite{chung2016deep}.
Some of the  well-known nonlinear activation functions are:

\begin{align}
sigmoid(x) &= 1/(1 + e^{-X})&\\
tanh(x) &= (1- e^{-2X})/(1+e^{-2X})\\ 
ReLU(x) &=\max (x; 0)
\end{align}
 The activation function in equation (3) Rectified Linear Unit (ReLU), is the most popular and widely-used activation function; and while some hand-designed activation functions have been introduced to replace ReLU, none have gained to popularity that ReLu has \cite{maas2013rectifier}\cite{clevert2015fast}
\cite{klambauer2017self}\cite{hahnloser2000digital} \cite{jarrett2009best} \cite{nair2010rectified}.

Trainable nonlinear activation functions have been  proposed by \cite{chung2016deep}, \cite{temurtas2004study}. 
Chung \textit{et al.} \cite{chung2016deep} used a Taylor series approximation of $sigmoid$, $tanh$, and $ReLU$ as an initialization point for their activation functions, and trained the coefficients of the Taylor series approximation to optimize training. This implementation used the same polynomial function on each neuron of a given layer. The results were comparable to the state of the art. 

In this manuscript, we present a type of neural network in which the activation functions in every layer form a polynomial basis, i.e. groups of neurons are assigned to  unique monomials in a given layer. We also propose a new architecture in which we vertically concatenate many fully connected layers to form one layer that makes computation more efficient. We do not train activation functions. Our activation functions are fixed and they form a polynomial basis. The structure of the hidden layers follows the pattern of: (i) a layer with  polynomials as the activation functions, and (ii) a layer with a linear activation function.   

The remaining of this manuscript is organized as follows: Section 2 describes the mathematical foundations and the architecture for SWAG, Section 3 describes the experiments that were conducted, and section 4 is a discussion of results and future work.



\section{Methods}

\subsection{Representation of functions with a basis}
Suppose that we have a data set $\{\mathbf{x}_j\}$ for $1 \leq j \leq n$ and labels $\{\mathbf{y}_j\}$ that corresponding to our data set. We would like to find a function $f(x)$ such that $f(x_j)=y_j$ for all $1\leq j\leq n$. The Stone-Weierstrass approximation theorem states that any continuous real valued function on a compact set can be uniformly approximated by a polynomial. Formally:   
\begin{theorem}[Stone-Weierstrass Approximation Theorem]
Suppose $f$ is a continuous real-valued function defined on any closed and bounded subset $X\in \R ^m$ for any $m\in \mathbb{N}$. For every $\epsilon >0$, there exists a polynomial $p(x_1,x_2,\dots ,x_m)$ such that $|f(x_1,x_2,\dots ,x_m)-p(x_1,x_2,\dots ,x_m)|<\epsilon$ for any $(x_1,x_2,\dots ,x_m)\in X$ 
\end{theorem}
 The simplicity of polynomial systems make them very attractive analytically and computationally. They are easy to form and have well-understood properties. The use of polynomials of a given degree as activation functions for all neurons in a single layer seems to be mathematically discouraged in traditional neural network settings because they are not universal approximators. Particularly, Leshno \textit{et al}. (1993) \cite{leshno_multilayer_1993} proved the following theorem:
\begin{theorem}
Let $M$ be the set of functions which are $L_{loc}^\infty(\R)$ with the property that  the closure of the set of points of discontinuity of any function in $M$ has zero Lebesgue measure. Let $\sigma\in M$. Then for a fixed $x\in\R ^n$, $$span\{\sigma\{w \cdot x+\Theta\}:w\in\R ^n, \Theta\in\R\}$$ is dense in $\mathbb{C}(\R ^n)$ if and only if $\sigma$ is not an algebraic polynomial (a.e.)
\end{theorem}

This theorem implies that fully connected feedforward neural networks with a sufficient number of neurons are universal approximators if and only if the activation functions are not polynomials. We note that in this traditional setting it is assumed that the activation function is the same for every neuron in a given layer. We now give the following extension of the Stone-Weierstrass approximation theorem 
\begin{corollary}
Let $\sigma_{p}=\frac{x^p}{p!}$ for $0\leq p< \infty$. Then 
  $$span\{\sigma_p\{w \cdot x+\Theta\}:w\in\R ^n,\Theta\in\R\}$$ is dense in $\mathbb{C}(X^n)$ where $X^n\in\R ^n$ is a compact set.
\end{corollary}
\begin{proof}
Notice that $\{\sigma_{p}\}_{p=0}^\infty$ is a basis for the vector space of polynomials over $\R$. So since we know that polynomials are dense in $\mathbb{C}(X^n)$ by the Stone-Weierstrass approximation theorem the result follows.  
\end{proof}

This corollary implies that if we allow a diverse set of polynomial activation functions in a particular layer we will still have the result of universal approximation capabilities of feedfoward neural networks. Using the same framework as Leshno \textit{et al}. (1993) \cite{leshno_multilayer_1993}, in which the output was assumed to be in $\R^n$, an extension to higher dimensions can be easily obtained by re-defining $\sigma_{p}\{w\}$ as a pointwise operation that takes  each element of $w$ and raises it to the $p^{th}$ power, e.g. given $w = [2,3]$, then $\sigma_4 \{w\} = [2^4, 3^4]$.

\subsection{Architecture of the SWAG Algorithm}
Let $x_j\in \R ^d$ be a data point in our data set $\{x_j\}_{j=1}^{n}$.
\begin{itemize}
    \item[0] Normalize data to be in the interval [0,1]. 
    \item[1] Create the first polynomial layer as follows:
    \begin{itemize}
        \item[1.1] Choose a $k$ for the number of of polynomial terms used ($k$ is a hyperparameter of the model). 
        \item[1.2] Choose $l$ for the number of neurons that correspond to each monomial of the $1^{st}$ layer ($l$ is a hyperparameter of the model). 
        \item[1.3] Create $k$ fully-connected layers with $l$ neurons in each layer, all with $x_j$ as their inputs.
        \begin{itemize}
            \item[1.3.1] The $p^{th}$ fully-connected layer for $1\leq p\leq k$ is defined by $\sigma_p\{W  x+b\}$ for $W\in l\times d$, $b\in \R ^l$, and $\sigma_p$ as defined above.
            \item[1.3.2] Initialization of weights are random and drawn from $\mathcal{N}(0,1)$, a Gaussian distribution with mean 0 and standard deviation 1.  
        \end{itemize}
        \item[1.4] Vertically concatenate the $k$ layers to form a vector of length $l \cdot k$
    \end{itemize}
    \item[2] Create a layer with a linear activation function. This is considered the second layer of SWAG. 
    \item[3] To add a third and fourth layer,  repeat the structure of the previous 2 layers with the input of the third layer as the the output of the second layer. If a third and fourth layer is added then the first dimension of the matrix used in the second layer is a hyperparameter of the model. 
    \item[4] Continue to add layers in this pattern as desired. 
    \item[5] The matrix used for the final linear activation layer will have its first dimension be the dimension of the output vector.
\end{itemize}

Figure \ref{fig:Graph1} is a diagram of an example of SWAG using two layers and Figure  \ref{fig:Graph2} is a diagram of an example of SWAG with four layers.

\begin{figure}[h]\label{fig:Our Model}
    \centering
    \includegraphics[width=0.7\textwidth]{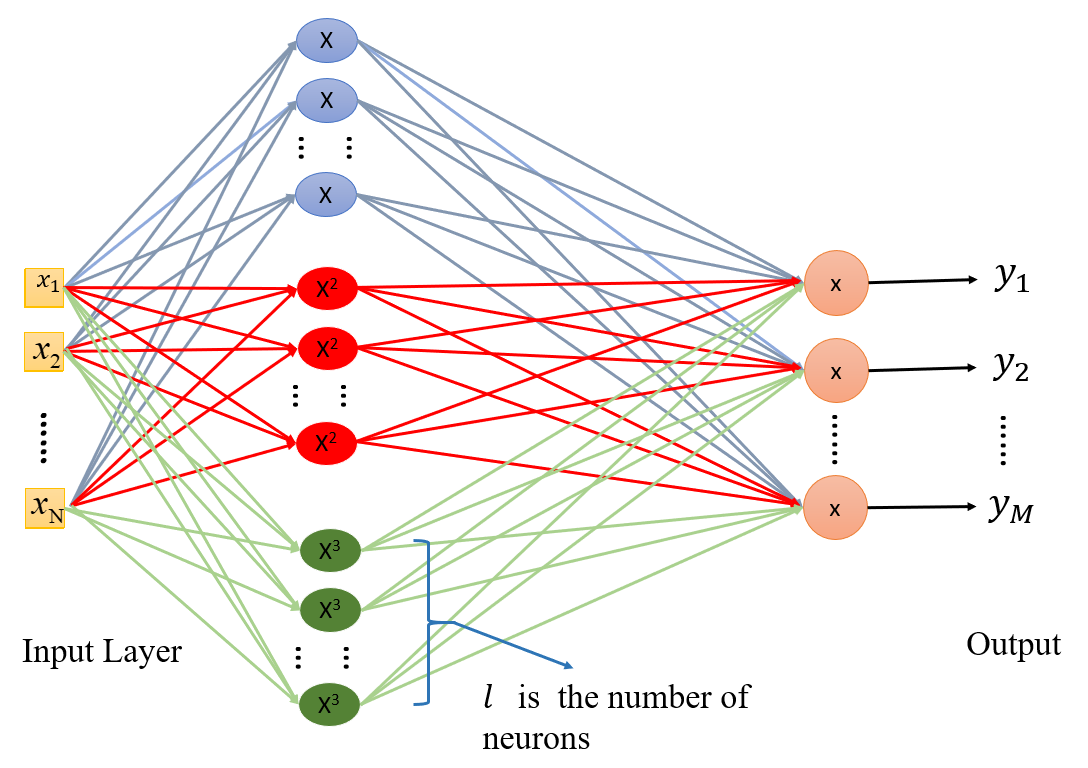}
    \caption{ Implementation of the SWAG architecture with three groups of monomials of powers 1 through 3,  and two layers}
    \label{fig:Graph1}
\end{figure}
\begin{figure}[h]\label{fig:Our Model1}
    \centering
    \includegraphics[width=0.7\textwidth]{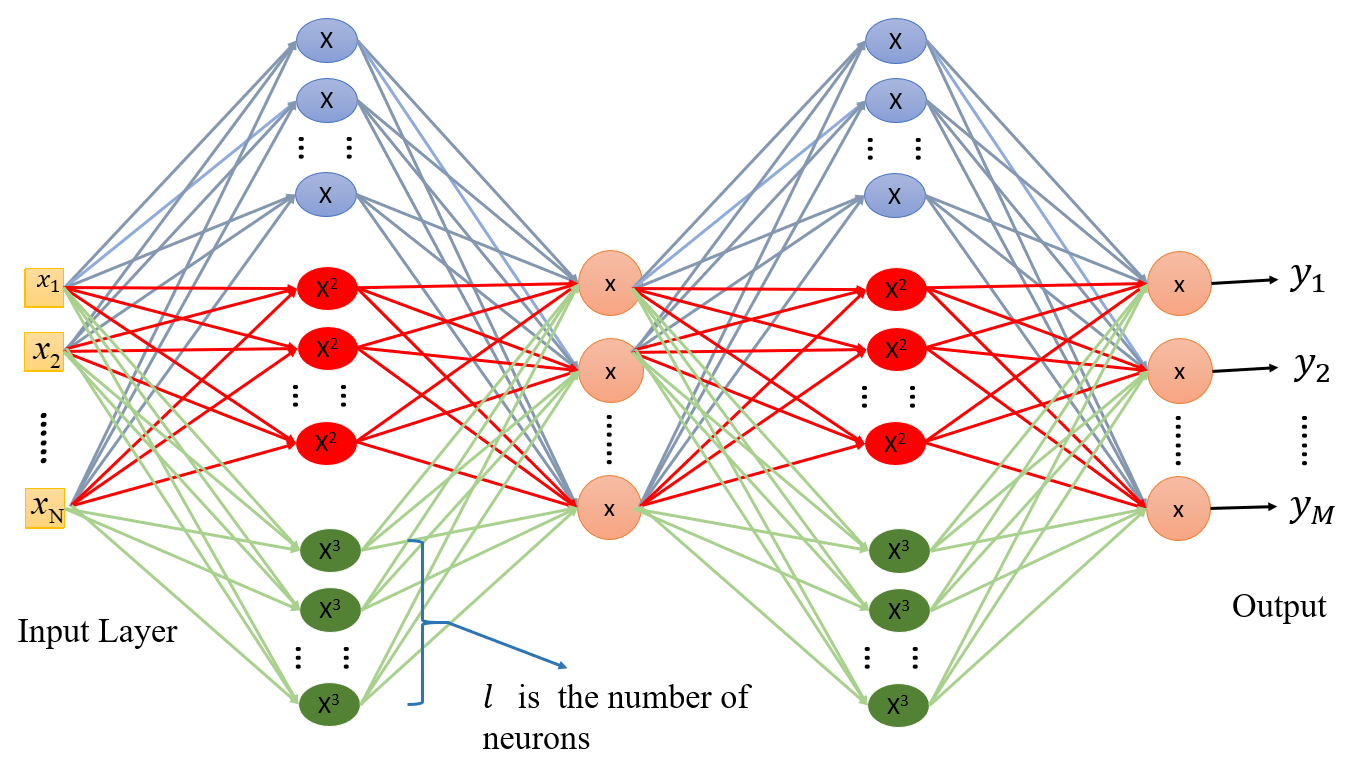}
    \caption{  Implementation of the SWAG architecture with three groups of monomials of powers 1 through 3, and four layers}
    \label{fig:Graph2}
\end{figure}

\section{ Results }

\subsection{Representation of Non-Linear Functions}

To test our model, we generated a random data set $X_{train} = \{x_j\}_{j=1}^{1000}$, with $x_j \in (0,1)$ as the vector for training and $X_{test} = \{x_k\}_{k=1}^{200}$, with $x_k \in (0,1)$ as the vector for testing. We selected three functions for which traditional DNNs do not converge at all, or require a number of epochs orders of magnitude larger than SWAG to converge.   
\begin{align}
    F_{1} &= \frac{1}{2}x^{2}-  5\left(\frac{\mathrm{1} }{\mathrm{1} + e^{x} }\right)\\
    F_{2} &=6 x^{5}- 3 \left(\frac{\mathrm{1} }{\mathrm{1} + e^{x} } \right) +e^{x} - 9 \ {\log_{10} (x)}\\
    F_{3} &=22  x^{20}-  \frac{\mathrm{1} }{\mathrm{1} + e^{x} } + 2 \ e^{x}+5 \ {\log_{10} (x)}
\end{align}

$$ 1\leq i\leq 3    \qquad   Y_{i_{train}}= F_{i}(X_{train})       \qquad    \qquad Y_{i_{test}}= F_{i}(X_{test})$$



We trained 5 traditional DNNs of various architectures (code in appendix). We also trained SWAG, with $l=50$, $k=8$, and we used 4 layers. The first dimension of the second layer in this implementation of SWAG was $50$. We used the standard mean squared loss function with Adam optimizer to test model accuracy \cite{kingma2014adam}. 

We conducted a first experiment for $F_1$, as shown in Figure \ref{fig:Graph2}. SWAG is the only model that had the cost function converge to zero after 50 epochs of training on $F_1$. We also note that Figure \ref{fig:Graph3} gives a visual representation of how the different architectures reconstructed $F_1$. 

We conducted a second experiment with $X_{train}=\{ 0.01,0.02,0.03,\hdots,1 \} $ and \\$X_{test}=\{ 0.015,0.025,0.035\hdots,0.985 \}$ . This allows the test and training sets to have almost the same number of points.
$$ 1\leq i\leq 3    \qquad   Y_{i_{train}}= F_{i}(x_{train})       \qquad   \qquad Y_{1_{test}}= F_{1}(x_{test})$$
We repeat the process of the first experiment to train and test the various models. The results of the two experiments are found in Figures 5-14 in the appendix. 
\begin{figure}[t]
\centering
\includegraphics[width=0.9\textwidth]{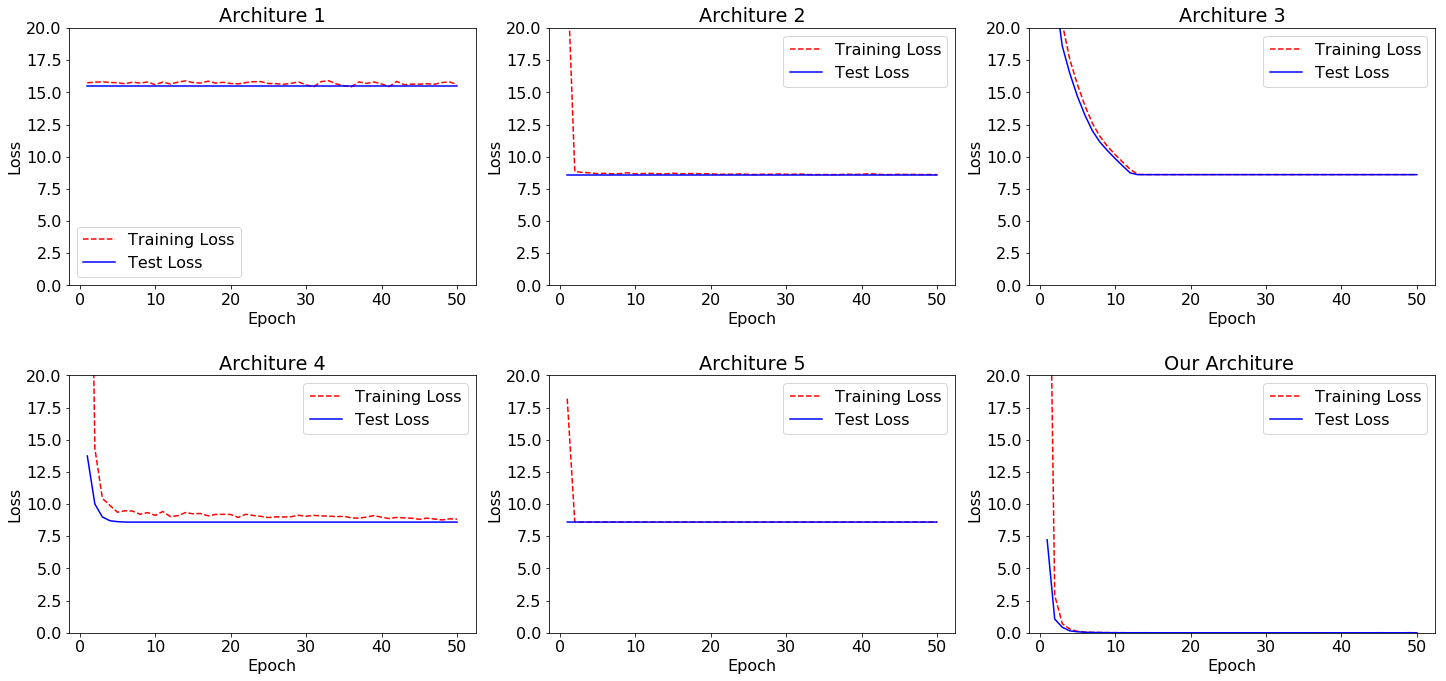}
\caption{ $F_{1} = \frac{1}{2}x^{2}-  5\left(\frac{\mathrm{1} }{\mathrm{1} + e^{x} }\right) $ Experiment 1 shape }
\label{fig:Graph2}
\end{figure}
\begin{figure}[h!]
\centering
\includegraphics[width=0.9\textwidth]{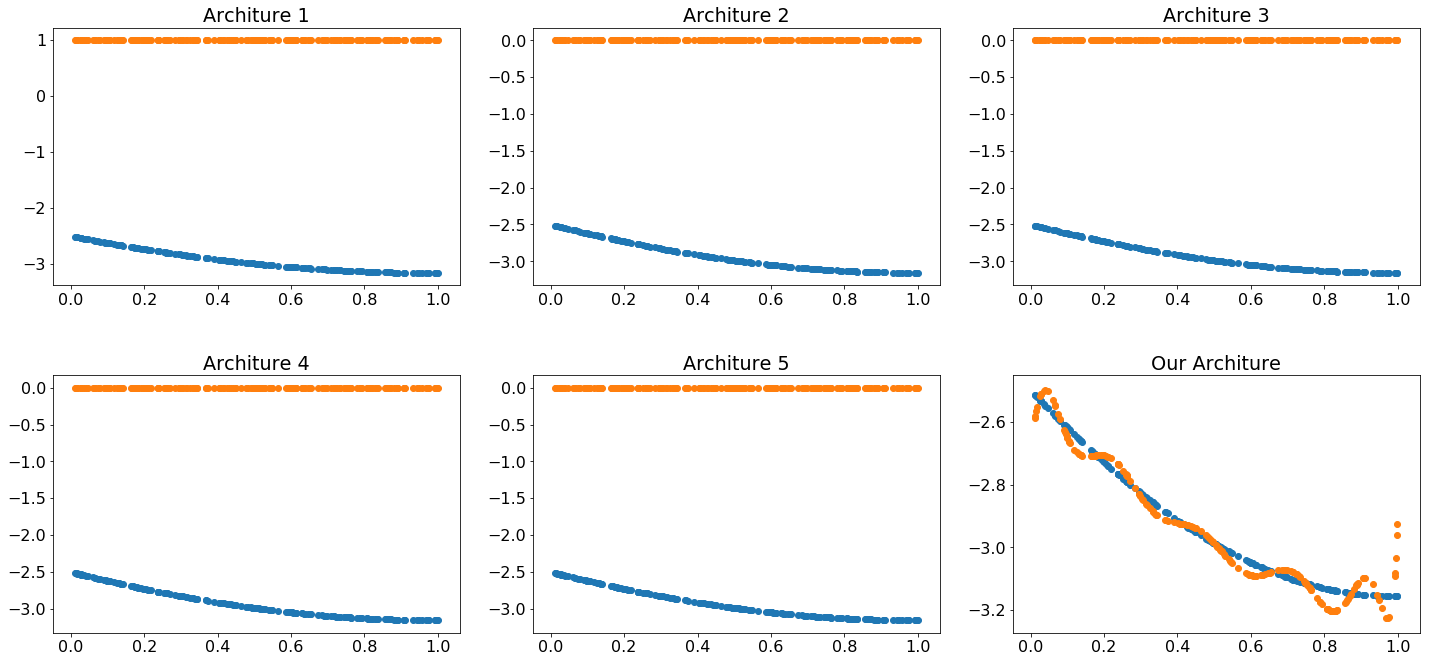}
\caption{ $F_{1} = \frac{1}{2}x^{2}-  5\left(\frac{\mathrm{1} }{\mathrm{1} + e^{x} }\right) $ Experiment 1 shape }
\label{fig:Graph3}
\end{figure}
\subsection{MNIST Handwriting Data Set}

For our final experiment we ran SWAG on the  MNIST hand writing  data set \cite{lecun-mnisthandwrittendigit-2010}. The data set is composed of a total of 70,000 images, all of which are unique hand writing samples of the numbers 0-9. We flattened these images into vectors of size $(784,1)$ and used these as inputs to a traditional DNN, as well as SWAG. The traditional DNN had three hidden layer. In the first and second layers we used $ReLU$ as the activation function with 1024 neurons in each layer. For the third layer we used $Softmax$ as the activation function with 10 neurons (code in appendix). For our implementation of SWAG, we used $l=500$, $k=7$ and 2 layers. We used a training set that consisted of 60,000 images, and a test set that consisted of 10,000 images. 
In the traditional method we got a test accuracy of $0.9767$ after 4 epochs. SWAG achieved $0.9787$ test accuracy after 4 epochs. The results are shown in the appendix in Figure 15 and Figure 16.

\section{ Discussion }

In this work, we introduced a set of activation functions and a new architecture. We named this architecture SWAG. 
The first layer of our architecture has at least $k$ neurons where $k$ is the degree of the  polynomial for estimation of the function $g(x)$ such that $g(x_j)=y_j$ for all $1\leq j\leq n$; this layer has $k$ different activation functions $\{\frac{x^p}{p!}\}_{p=1}^k$. The second layer is a fully connected layer with a linear activation function. To add additional layers the pattern of the first 2 layers is repeated. By using the back propagation algorithm we can find the set of weights that optimize the predictions.

We created a random data set with highly complicated nonlinear functions. We evaluated the effectiveness of SWAG and found that it was able to approximate the functions better than traditional deep learning methods; it also converged faster. Finally we tested SWAG on the MNIST handwriting data set. Our method was able to replicate the state of the art in fully-connected architectures while converging in only 4 epochs. 

We note that there are many basis sets that are able to estimate a function with arbitrary accuracy.  In future work, it will be important to compare the performance of different basis sets and function approximations to determine which one has better performance in specific situations. Our conjecture is that orthogonal basis will provide an advantage in some cases. Another interesting question to be pursued is to find a way to set the initial weights of the system more effectively. We believe that a Taylor estimation of our data set will increase the performance of SWAG after initialization. 

In addition we find the question of how to implement this architecture in convolutional and recursive neural networks especially interesting. Convolutional neural networks have surpassed the accuracy achieved with fully connected neural networks on the MNIST data set, and also reduce the number of parameters that are necessary to train a fully connected neural network. We also reduced the number of parameters in a fully connected network, but we were not able to surpass the state of the art in convolutional neural networks with our current implementation. We hypothesize that implementing the SWAG framework into convolutional and Recursive neural networks will allow us to further reduce parameters, make our model converge even faster and get a better accuracy then that which is currently possible.

\newcommand{\etalchar}[1]{$^{#1}$}

\newpage
\section{Appendix}

\begin{figure}[h!]
\centering
\includegraphics[width=0.9\textwidth]{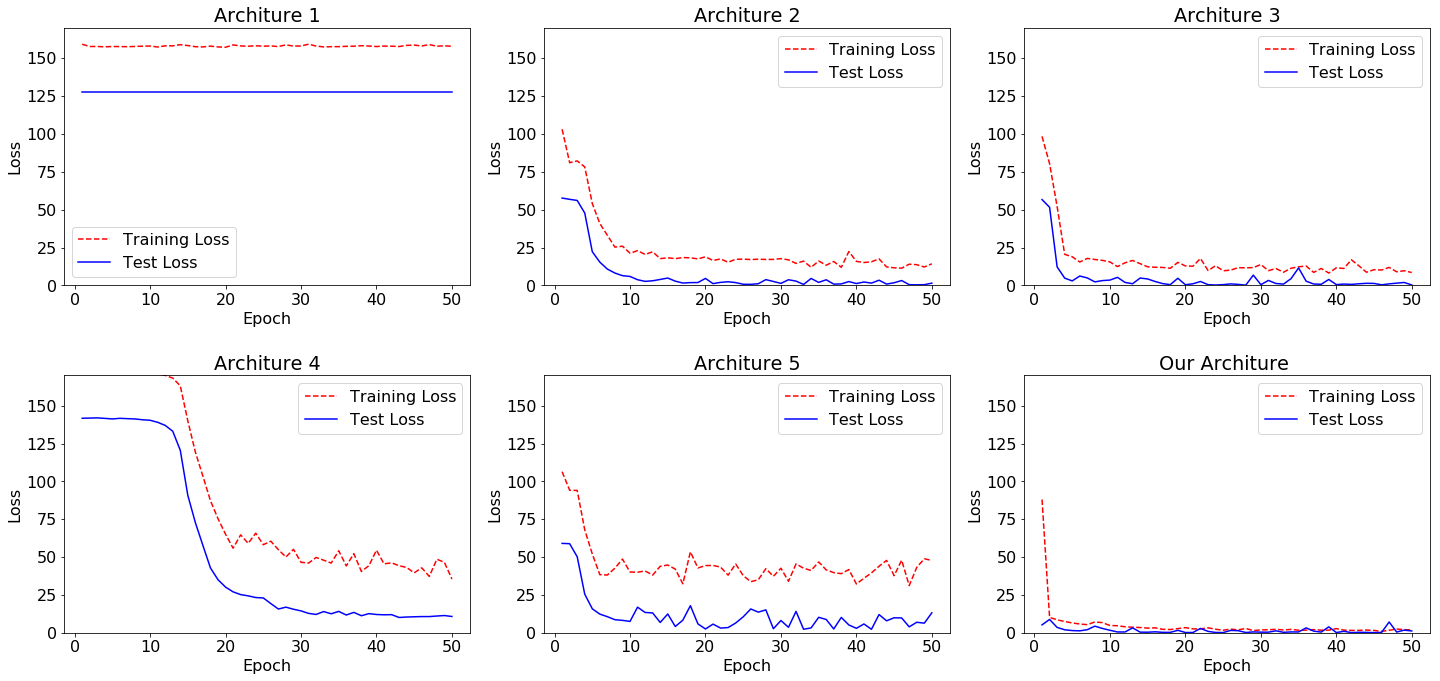}
\caption{\label{fig:Graph} $F_{2} =6 x^{5}- 3 \left(\frac{\mathrm{1} }{\mathrm{1} + e^{x} } \right) +e^{x} - 9 \ {\log_{10} (x)}$ Experiment 1 loss }
\end{figure}

\begin{figure}[h!]
\centering
\includegraphics[width=0.9\textwidth]{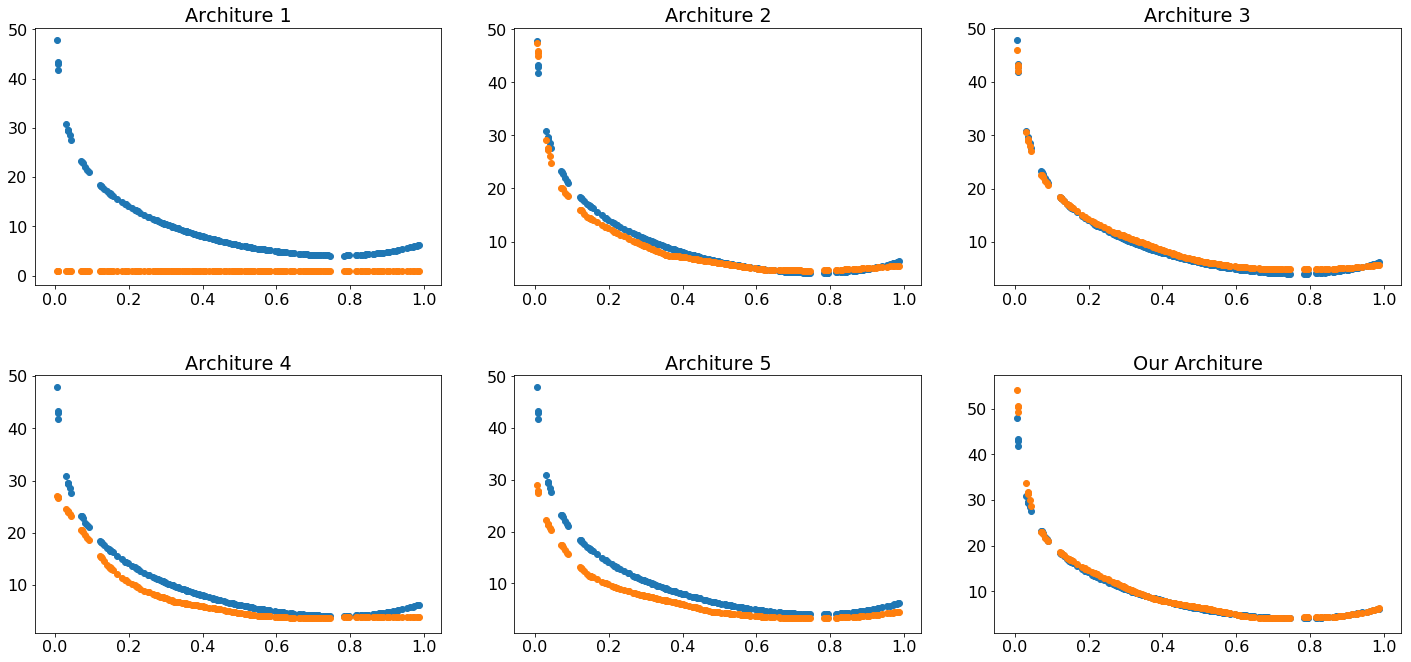}
\caption{\label{fig:Graph} $F_{2} =6 x^{5}- 3 \left(\frac{\mathrm{1} }{\mathrm{1} + e^{x} } \right) +e^{x} - 9 \ {\log_{10} (x)}$ Experiment 1 shape }
\end{figure}

\begin{figure}[h!]
\centering
\includegraphics[width=0.9\textwidth]{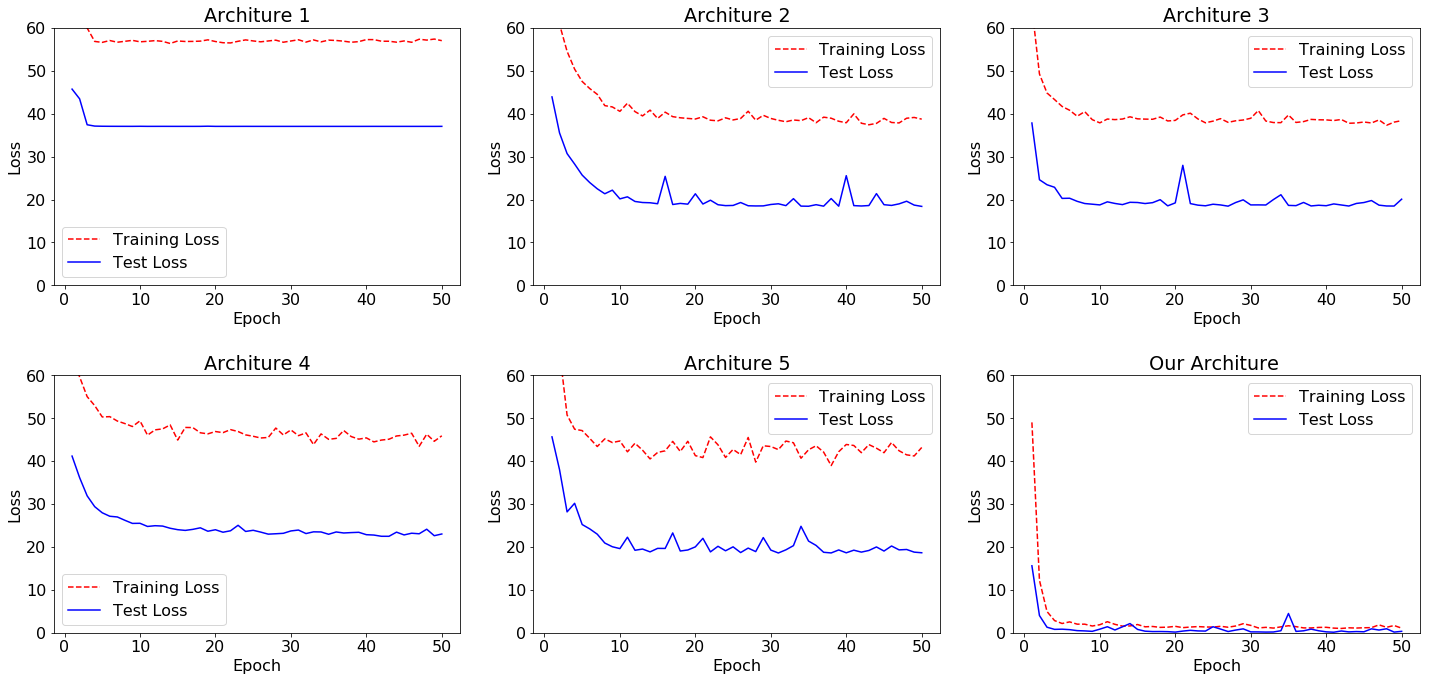}
\caption{\label{fig:Graph} $F_{3} =22  x^{20}-  \frac{\mathrm{1} }{\mathrm{1} + e^{x} } + 2 \ e^{x}+5 \ {\log_{10} (x)}$ Experiment 1 loss }
\end{figure}

\begin{figure}[h!]
\centering
\includegraphics[width=0.9\textwidth]{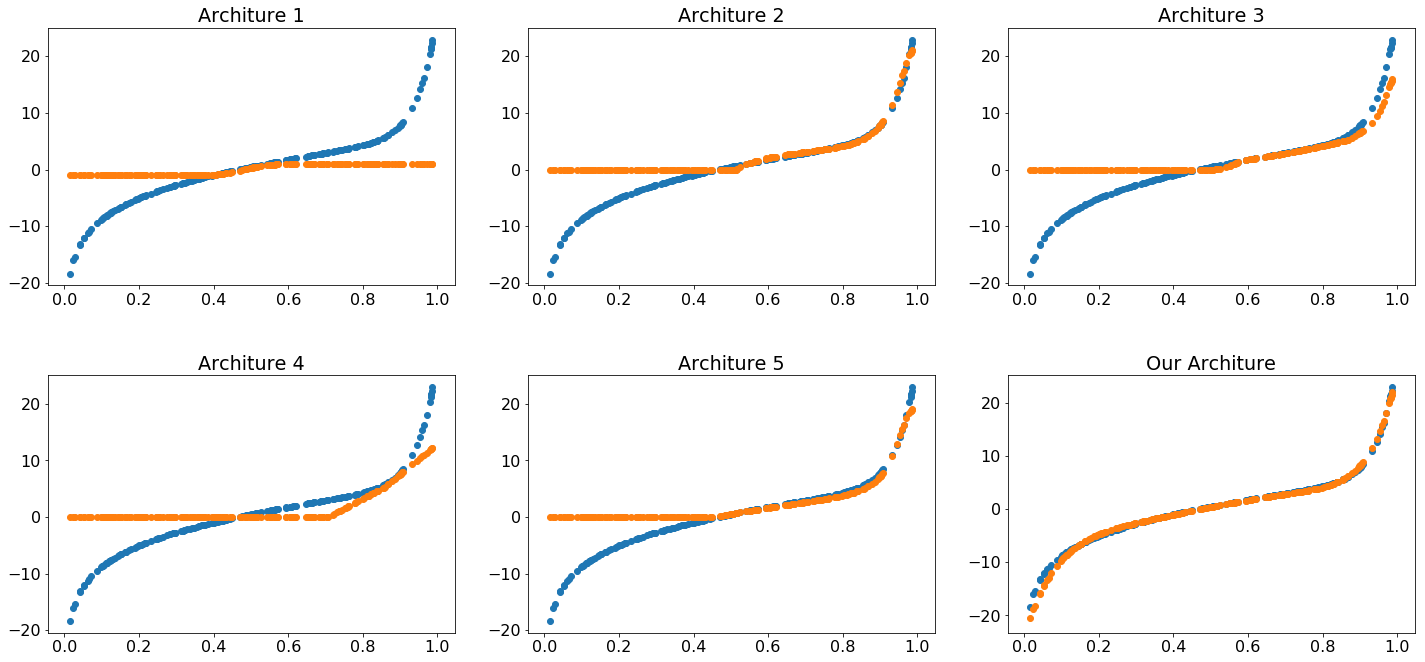}
\caption{\label{fig:Graph} $F_{3} =22  x^{20}-  \frac{\mathrm{1} }{\mathrm{1} + e^{x} } + 2 \ e^{x}+5 \ {\log_{10} (x)}$ Experiment 1 shape}
\end{figure}

\begin{figure}[h!]
\centering
\includegraphics[width=0.9\textwidth]{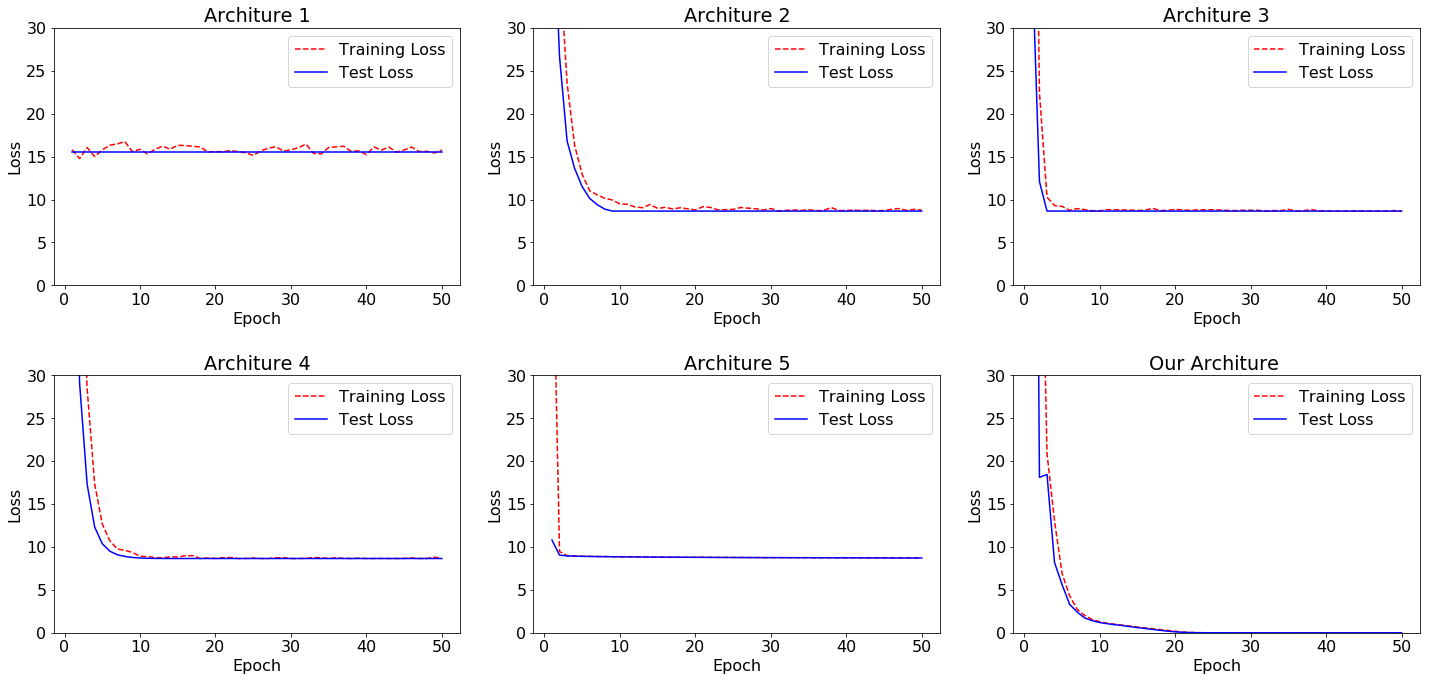}
\caption{\label{fig:Graph} $F_{1} = \frac{1}{2}x^{2}-  5\left(\frac{\mathrm{1} }{\mathrm{1} + e^{x} }\right)$ Experiment 2 loss }
\end{figure}

\begin{figure}[h!]
\centering
\includegraphics[width=0.9\textwidth]{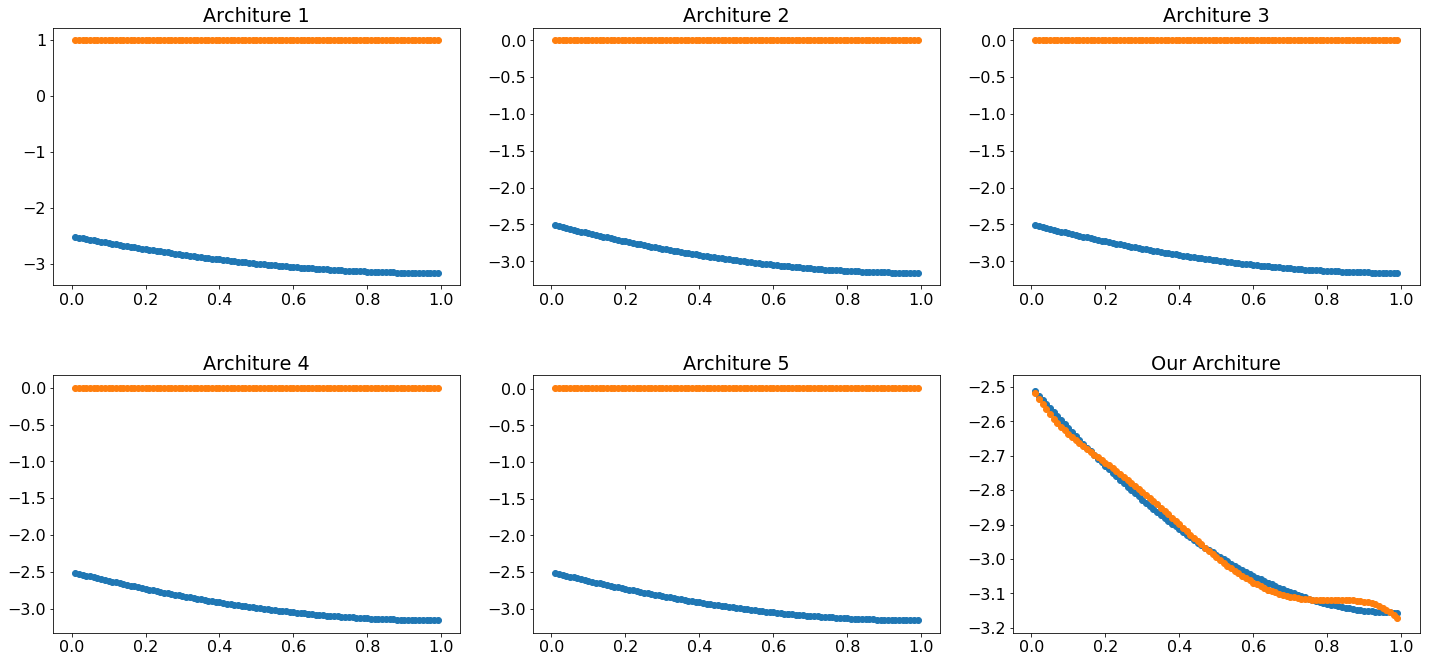}
\caption{\label{fig:Graph} $F_{1} = \frac{1}{2}x^{2}-  5\left(\frac{\mathrm{1} }{\mathrm{1} + e^{x} }\right)$ Experiment 2 shape }
\end{figure}

\begin{figure}[h!]
\centering
\includegraphics[width=0.9\textwidth]{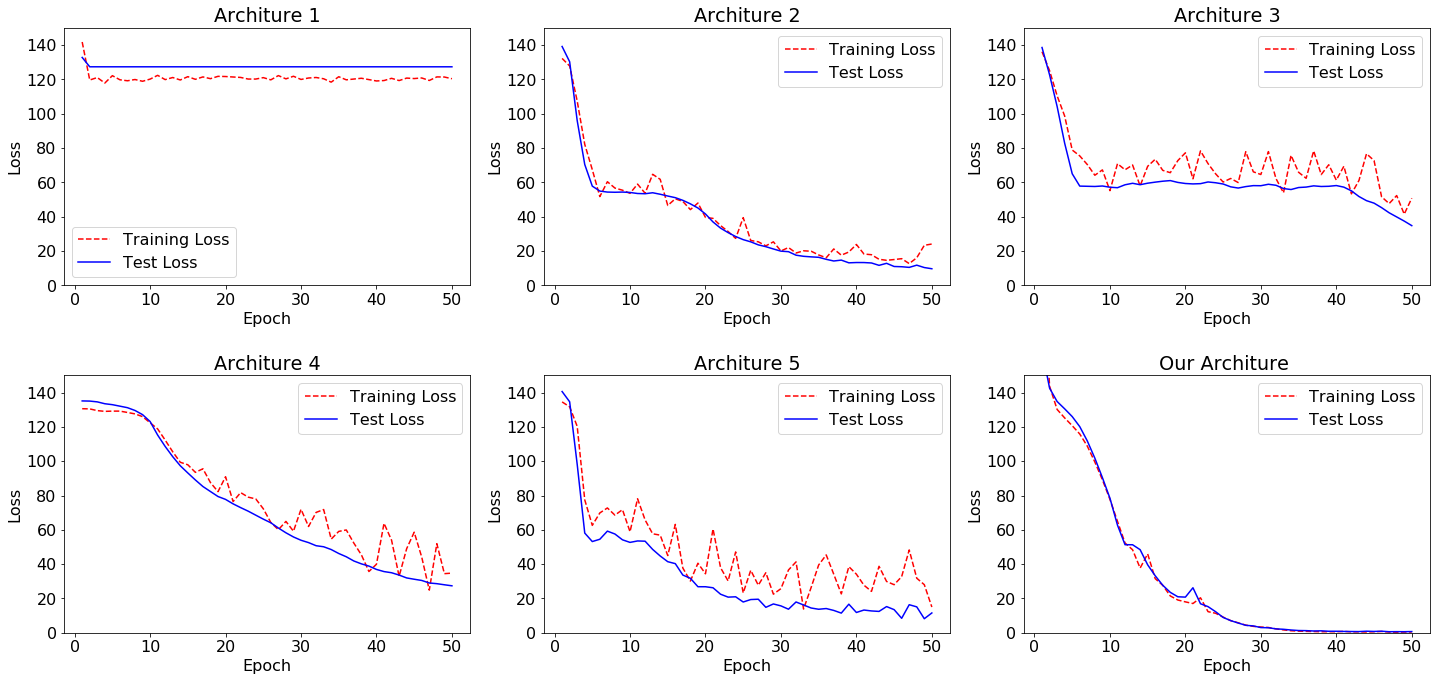}
\caption{\label{fig:Graph} $F_{2} =6 x^{5}- 3 \left(\frac{\mathrm{1} }{\mathrm{1} + e^{x} } \right) +e^{x} - 9 \ {\log_{10} (x)}$ Experiment 2 loss }
\end{figure}

\begin{figure}[h!]
\centering
\includegraphics[width=0.9\textwidth]{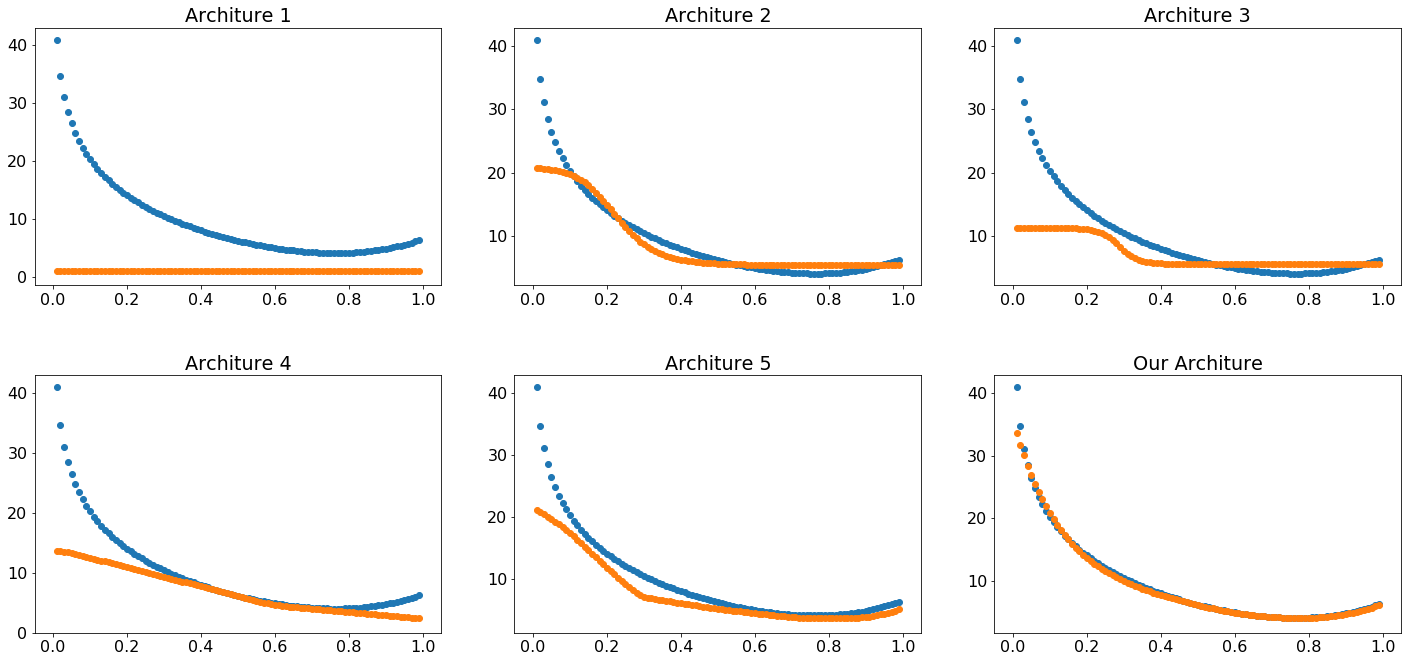}
\caption{\label{fig:Graph} $F_{2} =6 x^{5}- 3 \left(\frac{\mathrm{1} }{\mathrm{1} + e^{x} } \right) +e^{x} - 9 \ {\log_{10} (x)}$ Experiment 2 shape}
\end{figure}

\begin{figure}[h!]
\centering
\includegraphics[width=0.9\textwidth]{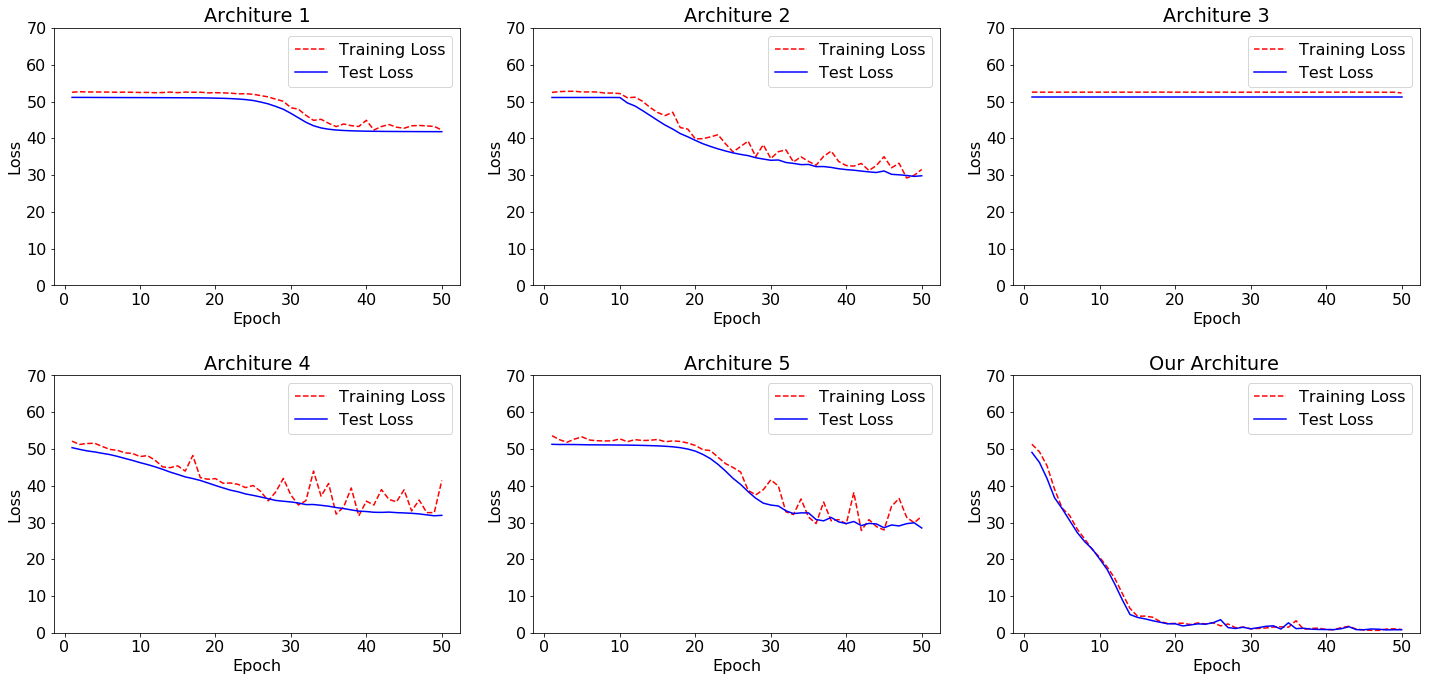}
\caption{\label{fig:Graph} $F_{3} =22  x^{20}-  \frac{\mathrm{1} }{\mathrm{1} + e^{x} } + 2 \ e^{x}+5 \ {\log_{10} (x)}$ Experiment 2 loss }
\end{figure}

\begin{figure}[h!]
\centering
\includegraphics[width=0.9\textwidth]{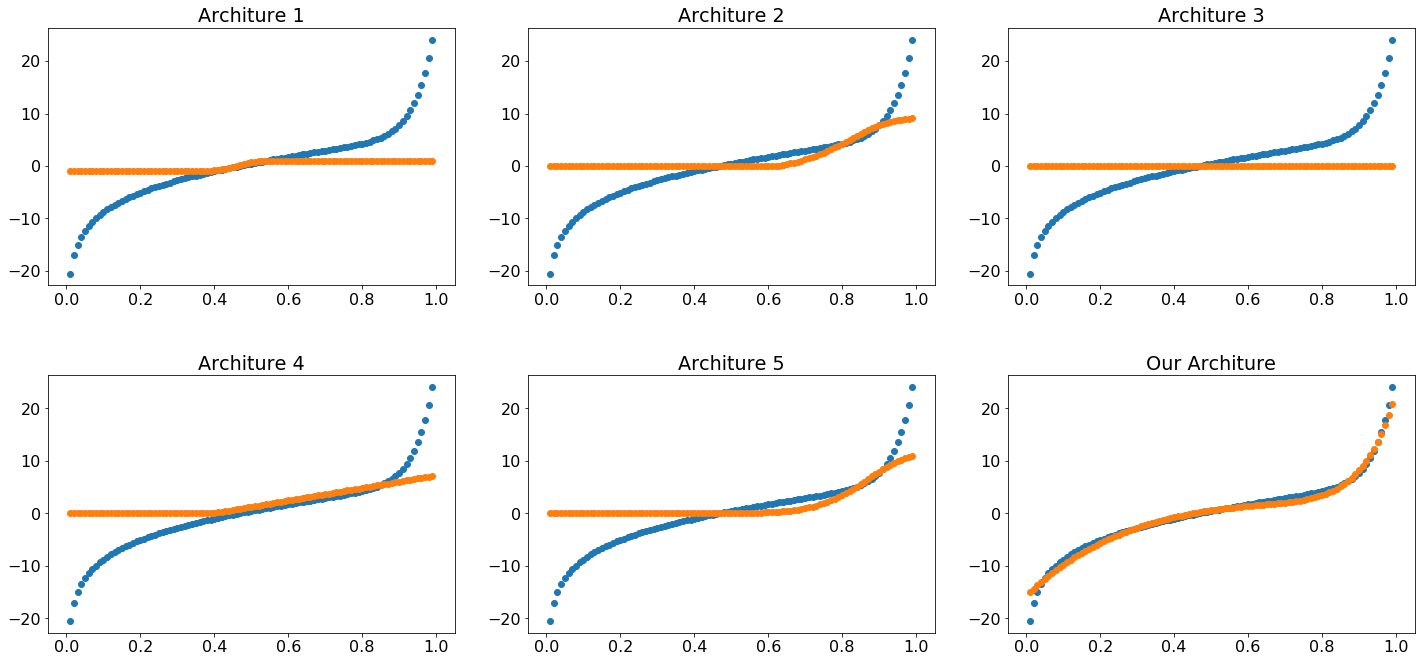}
\caption{\label{fig:Graph} $F_{3} =22  x^{20}-  \frac{\mathrm{1} }{\mathrm{1} + e^{x} } + 2 \ e^{x}+5 \ {\log_{10} (x)}$ Experiment 2 shape}
\end{figure}

\begin{figure}[h!]
\centering
\includegraphics[width=0.5\textwidth]{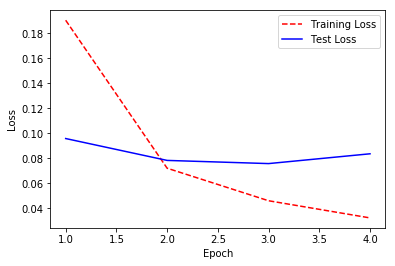}
\caption{\label{fig:2} Traditional Deep Learning on MNIST Data Set. Test loss: 0.08366 Test accuracy: 0.9767 }
\end{figure}

\begin{figure}[h!]
\centering
\includegraphics[width=0.5\textwidth]{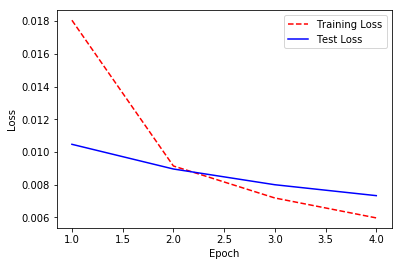}
\caption{\label{fig:graph} SWAG on MNIST Data Set. Test loss: 0.07297 Test accuracy: 0.9787 }
\end{figure}
\newpage

\include{python_code}

\subsection{Source Code}
All the source code is available at the following link:  \url{https://github.com/DeepLearningSaeid/New-Type-of-Deep-Learning/}

\begin{Verbatim}[commandchars=\\\{\}]
{\color{incolor}In [{\color{incolor}23}]:} \PY{n}{model} \PY{o}{=} \PY{n}{Sequential}\PY{p}{(}\PY{p}{)}
         \PY{n}{model}\PY{o}{.}\PY{n}{add}\PY{p}{(}\PY{n}{Dense}\PY{p}{(}\PY{l+m+mi}{10}\PY{p}{,} \PY{n}{input\PYZus{}dim}\PY{o}{=}\PY{n}{input\PYZus{}dim}\PY{p}{,} \PY{n}{activation}\PY{o}{=}\PY{l+s+s1}{\PYZsq{}}\PY{l+s+s1}{relu}\PY{l+s+s1}{\PYZsq{}}\PY{p}{)}\PY{p}{)}
         \PY{n}{model}\PY{o}{.}\PY{n}{add}\PY{p}{(}\PY{n}{Dense}\PY{p}{(}\PY{l+m+mi}{20}\PY{p}{,} \PY{n}{activation}\PY{o}{=}\PY{l+s+s1}{\PYZsq{}}\PY{l+s+s1}{sigmoid}\PY{l+s+s1}{\PYZsq{}}\PY{p}{)}\PY{p}{)}
         \PY{n}{model}\PY{o}{.}\PY{n}{add}\PY{p}{(}\PY{n}{Dense}\PY{p}{(}\PY{l+m+mi}{30}\PY{p}{,} \PY{n}{activation}\PY{o}{=}\PY{l+s+s1}{\PYZsq{}}\PY{l+s+s1}{tanh}\PY{l+s+s1}{\PYZsq{}}\PY{p}{)}\PY{p}{)}
         \PY{n}{model}\PY{o}{.}\PY{n}{add}\PY{p}{(}\PY{n}{Dense}\PY{p}{(}\PY{l+m+mi}{20}\PY{p}{,} \PY{n}{activation}\PY{o}{=}\PY{l+s+s1}{\PYZsq{}}\PY{l+s+s1}{relu}\PY{l+s+s1}{\PYZsq{}}\PY{p}{)}\PY{p}{)}
         \PY{n}{model}\PY{o}{.}\PY{n}{add}\PY{p}{(}\PY{n}{Dense}\PY{p}{(}\PY{l+m+mi}{15}\PY{p}{,} \PY{n}{activation}\PY{o}{=}\PY{l+s+s1}{\PYZsq{}}\PY{l+s+s1}{sigmoid}\PY{l+s+s1}{\PYZsq{}}\PY{p}{)}\PY{p}{)}
         \PY{n}{model}\PY{o}{.}\PY{n}{add}\PY{p}{(}\PY{n}{Dense}\PY{p}{(}\PY{l+m+mi}{25}\PY{p}{,} \PY{n}{activation}\PY{o}{=}\PY{l+s+s1}{\PYZsq{}}\PY{l+s+s1}{relu}\PY{l+s+s1}{\PYZsq{}}\PY{p}{)}\PY{p}{)}
         \PY{n}{model}\PY{o}{.}\PY{n}{add}\PY{p}{(}\PY{n}{Dense}\PY{p}{(}\PY{l+m+mi}{10}\PY{p}{,} \PY{n}{activation}\PY{o}{=}\PY{l+s+s1}{\PYZsq{}}\PY{l+s+s1}{relu}\PY{l+s+s1}{\PYZsq{}}\PY{p}{)}\PY{p}{)}
         \PY{n}{model}\PY{o}{.}\PY{n}{add}\PY{p}{(}\PY{n}{Dense}\PY{p}{(}\PY{n}{output\PYZus{}dim}\PY{p}{,} \PY{n}{activation}\PY{o}{=}\PY{l+s+s1}{\PYZsq{}}\PY{l+s+s1}{tanh}\PY{l+s+s1}{\PYZsq{}}\PY{p}{)}\PY{p}{)}
         \PY{n}{model}\PY{o}{.}\PY{n}{add}\PY{p}{(}\PY{n}{Dropout}\PY{p}{(}\PY{l+m+mf}{0.2}\PY{p}{)}\PY{p}{)}
         \PY{n}{model}\PY{o}{.}\PY{n}{summary}\PY{p}{(}\PY{p}{)}
         \PY{n}{model}\PY{o}{.}\PY{n}{compile}\PY{p}{(}\PY{n}{loss}\PY{o}{=}\PY{l+s+s1}{\PYZsq{}}\PY{l+s+s1}{mean\PYZus{}squared\PYZus{}error}\PY{l+s+s1}{\PYZsq{}}\PY{p}{,} \PY{n}{optimizer}\PY{o}{=}\PY{l+s+s1}{\PYZsq{}}\PY{l+s+s1}{adam}\PY{l+s+s1}{\PYZsq{}}\PY{p}{)}
         \PY{n}{model}\PY{o}{.}\PY{n}{fit}\PY{p}{(}\PY{n}{train\PYZus{}x}\PY{p}{,}\PY{n}{train\PYZus{}y}\PY{p}{,}\PY{n}{epochs}\PY{o}{=}\PY{n}{number\PYZus{}epo}\PY{p}{,}\PY{n}{verbose}\PY{o}{=}\PY{l+m+mi}{0}\PY{p}{,}\PY{n}{batch\PYZus{}size}\PY{o}{=}\PY{l+m+mi}{10}\PY{p}{,}
                   \PY{n}{validation\PYZus{}data}\PY{o}{=}\PY{p}{(}\PY{n}{test\PYZus{}x}\PY{p}{,} \PY{n}{test\PYZus{}y}\PY{p}{)}\PY{p}{)} 
\end{Verbatim}

    \begin{Verbatim}[commandchars=\\\{\}]
\_\_\_\_\_\_\_\_\_\_\_\_\_\_\_\_\_\_\_\_\_\_\_\_\_\_\_\_\_\_\_\_\_\_\_\_\_\_\_\_\_\_\_\_\_\_\_\_\_\_\_\_\_\_\_\_\_\_\_\_\_\_\_\_\_
Layer (type)                 Output Shape              Param \#   
=================================================================
dense\_132 (Dense)            (None, 10)                20        
\_\_\_\_\_\_\_\_\_\_\_\_\_\_\_\_\_\_\_\_\_\_\_\_\_\_\_\_\_\_\_\_\_\_\_\_\_\_\_\_\_\_\_\_\_\_\_\_\_\_\_\_\_\_\_\_\_\_\_\_\_\_\_\_\_
dense\_133 (Dense)            (None, 20)                220       
\_\_\_\_\_\_\_\_\_\_\_\_\_\_\_\_\_\_\_\_\_\_\_\_\_\_\_\_\_\_\_\_\_\_\_\_\_\_\_\_\_\_\_\_\_\_\_\_\_\_\_\_\_\_\_\_\_\_\_\_\_\_\_\_\_
dense\_134 (Dense)            (None, 30)                630       
\_\_\_\_\_\_\_\_\_\_\_\_\_\_\_\_\_\_\_\_\_\_\_\_\_\_\_\_\_\_\_\_\_\_\_\_\_\_\_\_\_\_\_\_\_\_\_\_\_\_\_\_\_\_\_\_\_\_\_\_\_\_\_\_\_
dense\_135 (Dense)            (None, 20)                620       
\_\_\_\_\_\_\_\_\_\_\_\_\_\_\_\_\_\_\_\_\_\_\_\_\_\_\_\_\_\_\_\_\_\_\_\_\_\_\_\_\_\_\_\_\_\_\_\_\_\_\_\_\_\_\_\_\_\_\_\_\_\_\_\_\_
dense\_136 (Dense)            (None, 15)                315       
\_\_\_\_\_\_\_\_\_\_\_\_\_\_\_\_\_\_\_\_\_\_\_\_\_\_\_\_\_\_\_\_\_\_\_\_\_\_\_\_\_\_\_\_\_\_\_\_\_\_\_\_\_\_\_\_\_\_\_\_\_\_\_\_\_
dense\_137 (Dense)            (None, 25)                400       
\_\_\_\_\_\_\_\_\_\_\_\_\_\_\_\_\_\_\_\_\_\_\_\_\_\_\_\_\_\_\_\_\_\_\_\_\_\_\_\_\_\_\_\_\_\_\_\_\_\_\_\_\_\_\_\_\_\_\_\_\_\_\_\_\_
dense\_138 (Dense)            (None, 10)                260       
\_\_\_\_\_\_\_\_\_\_\_\_\_\_\_\_\_\_\_\_\_\_\_\_\_\_\_\_\_\_\_\_\_\_\_\_\_\_\_\_\_\_\_\_\_\_\_\_\_\_\_\_\_\_\_\_\_\_\_\_\_\_\_\_\_
dense\_139 (Dense)            (None, 1)                 11        
\_\_\_\_\_\_\_\_\_\_\_\_\_\_\_\_\_\_\_\_\_\_\_\_\_\_\_\_\_\_\_\_\_\_\_\_\_\_\_\_\_\_\_\_\_\_\_\_\_\_\_\_\_\_\_\_\_\_\_\_\_\_\_\_\_
dropout\_17 (Dropout)         (None, 1)                 0         
=================================================================
Total params: 2,476
Trainable params: 2,476
Non-trainable params: 0
\_\_\_\_\_\_\_\_\_\_\_\_\_\_\_\_\_\_\_\_\_\_\_\_\_\_\_\_\_\_\_\_\_\_\_\_\_\_\_\_\_\_\_\_\_\_\_\_\_\_\_\_\_\_\_\_\_\_\_\_\_\_\_\_\_
Run Time : 15.135070

    \end{Verbatim}

\begin{Verbatim}[commandchars=\\\{\}]
{\color{outcolor}Out[{\color{outcolor}23}]:} [<matplotlib.lines.Line2D at 0x2440d33ea20>]
\end{Verbatim}

    \begin{Verbatim}[commandchars=\\\{\}]
{\color{incolor}In [{\color{incolor}24}]:} \PY{n}{model} \PY{o}{=} \PY{n}{Sequential}\PY{p}{(}\PY{p}{)}
         \PY{n}{model}\PY{o}{.}\PY{n}{add}\PY{p}{(}\PY{n}{Dense}\PY{p}{(}\PY{l+m+mi}{5}\PY{p}{,} \PY{n}{input\PYZus{}dim}\PY{o}{=}\PY{n}{input\PYZus{}dim}\PY{p}{,} \PY{n}{activation}\PY{o}{=}\PY{l+s+s1}{\PYZsq{}}\PY{l+s+s1}{relu}\PY{l+s+s1}{\PYZsq{}}\PY{p}{)}\PY{p}{)}
         \PY{n}{model}\PY{o}{.}\PY{n}{add}\PY{p}{(}\PY{n}{Dense}\PY{p}{(}\PY{l+m+mi}{10}\PY{p}{,} \PY{n}{activation}\PY{o}{=}\PY{l+s+s1}{\PYZsq{}}\PY{l+s+s1}{relu}\PY{l+s+s1}{\PYZsq{}}\PY{p}{)}\PY{p}{)}
         \PY{n}{model}\PY{o}{.}\PY{n}{add}\PY{p}{(}\PY{n}{Dense}\PY{p}{(}\PY{l+m+mi}{50}\PY{p}{,} \PY{n}{activation}\PY{o}{=}\PY{l+s+s1}{\PYZsq{}}\PY{l+s+s1}{tanh}\PY{l+s+s1}{\PYZsq{}}\PY{p}{)}\PY{p}{)}
         \PY{n}{model}\PY{o}{.}\PY{n}{add}\PY{p}{(}\PY{n}{Dense}\PY{p}{(}\PY{l+m+mi}{18}\PY{p}{,} \PY{n}{activation}\PY{o}{=}\PY{l+s+s1}{\PYZsq{}}\PY{l+s+s1}{relu}\PY{l+s+s1}{\PYZsq{}}\PY{p}{)}\PY{p}{)}
         \PY{n}{model}\PY{o}{.}\PY{n}{add}\PY{p}{(}\PY{n}{Dense}\PY{p}{(}\PY{l+m+mi}{15}\PY{p}{,} \PY{n}{activation}\PY{o}{=}\PY{l+s+s1}{\PYZsq{}}\PY{l+s+s1}{tanh}\PY{l+s+s1}{\PYZsq{}}\PY{p}{)}\PY{p}{)}
         \PY{n}{model}\PY{o}{.}\PY{n}{add}\PY{p}{(}\PY{n}{Dense}\PY{p}{(}\PY{l+m+mi}{18}\PY{p}{,} \PY{n}{activation}\PY{o}{=}\PY{l+s+s1}{\PYZsq{}}\PY{l+s+s1}{sigmoid}\PY{l+s+s1}{\PYZsq{}}\PY{p}{)}\PY{p}{)}
         \PY{n}{model}\PY{o}{.}\PY{n}{add}\PY{p}{(}\PY{n}{Dropout}\PY{p}{(}\PY{l+m+mf}{0.2}\PY{p}{)}\PY{p}{)}
         \PY{n}{model}\PY{o}{.}\PY{n}{add}\PY{p}{(}\PY{n}{Dense}\PY{p}{(}\PY{l+m+mi}{8}\PY{p}{,} \PY{n}{activation}\PY{o}{=}\PY{l+s+s1}{\PYZsq{}}\PY{l+s+s1}{relu}\PY{l+s+s1}{\PYZsq{}}\PY{p}{)}\PY{p}{)}
         \PY{n}{model}\PY{o}{.}\PY{n}{add}\PY{p}{(}\PY{n}{Dropout}\PY{p}{(}\PY{l+m+mf}{0.2}\PY{p}{)}\PY{p}{)}
         \PY{n}{model}\PY{o}{.}\PY{n}{add}\PY{p}{(}\PY{n}{Dense}\PY{p}{(}\PY{n}{output\PYZus{}dim}\PY{p}{,} \PY{n}{activation}\PY{o}{=}\PY{l+s+s1}{\PYZsq{}}\PY{l+s+s1}{relu}\PY{l+s+s1}{\PYZsq{}}\PY{p}{)}\PY{p}{)}
         \PY{n}{model}\PY{o}{.}\PY{n}{summary}\PY{p}{(}\PY{p}{)}
         \PY{n}{model}\PY{o}{.}\PY{n}{compile}\PY{p}{(}\PY{n}{loss}\PY{o}{=}\PY{l+s+s1}{\PYZsq{}}\PY{l+s+s1}{mean\PYZus{}squared\PYZus{}error}\PY{l+s+s1}{\PYZsq{}}\PY{p}{,} \PY{n}{optimizer}\PY{o}{=}\PY{l+s+s1}{\PYZsq{}}\PY{l+s+s1}{adam}\PY{l+s+s1}{\PYZsq{}}\PY{p}{)}
         \PY{n}{model}\PY{o}{.}\PY{n}{fit}\PY{p}{(}\PY{n}{train\PYZus{}x}\PY{p}{,}\PY{n}{train\PYZus{}y}\PY{p}{,}\PY{n}{epochs}\PY{o}{=}\PY{n}{number\PYZus{}epo}\PY{p}{,}\PY{n}{verbose}\PY{o}{=}\PY{l+m+mi}{0}\PY{p}{,}\PY{n}{batch\PYZus{}size}\PY{o}{=}\PY{l+m+mi}{10}\PY{p}{,}
                   \PY{n}{validation\PYZus{}data}\PY{o}{=}\PY{p}{(}\PY{n}{test\PYZus{}x}\PY{p}{,} \PY{n}{test\PYZus{}y}\PY{p}{)}\PY{p}{)} 
\end{Verbatim}

    \begin{Verbatim}[commandchars=\\\{\}]
\_\_\_\_\_\_\_\_\_\_\_\_\_\_\_\_\_\_\_\_\_\_\_\_\_\_\_\_\_\_\_\_\_\_\_\_\_\_\_\_\_\_\_\_\_\_\_\_\_\_\_\_\_\_\_\_\_\_\_\_\_\_\_\_\_
Layer (type)                 Output Shape              Param \#   
=================================================================
dense\_140 (Dense)            (None, 5)                 10        
\_\_\_\_\_\_\_\_\_\_\_\_\_\_\_\_\_\_\_\_\_\_\_\_\_\_\_\_\_\_\_\_\_\_\_\_\_\_\_\_\_\_\_\_\_\_\_\_\_\_\_\_\_\_\_\_\_\_\_\_\_\_\_\_\_
dense\_141 (Dense)            (None, 10)                60        
\_\_\_\_\_\_\_\_\_\_\_\_\_\_\_\_\_\_\_\_\_\_\_\_\_\_\_\_\_\_\_\_\_\_\_\_\_\_\_\_\_\_\_\_\_\_\_\_\_\_\_\_\_\_\_\_\_\_\_\_\_\_\_\_\_
dense\_142 (Dense)            (None, 50)                550       
\_\_\_\_\_\_\_\_\_\_\_\_\_\_\_\_\_\_\_\_\_\_\_\_\_\_\_\_\_\_\_\_\_\_\_\_\_\_\_\_\_\_\_\_\_\_\_\_\_\_\_\_\_\_\_\_\_\_\_\_\_\_\_\_\_
dense\_143 (Dense)            (None, 18)                918       
\_\_\_\_\_\_\_\_\_\_\_\_\_\_\_\_\_\_\_\_\_\_\_\_\_\_\_\_\_\_\_\_\_\_\_\_\_\_\_\_\_\_\_\_\_\_\_\_\_\_\_\_\_\_\_\_\_\_\_\_\_\_\_\_\_
dense\_144 (Dense)            (None, 15)                285       
\_\_\_\_\_\_\_\_\_\_\_\_\_\_\_\_\_\_\_\_\_\_\_\_\_\_\_\_\_\_\_\_\_\_\_\_\_\_\_\_\_\_\_\_\_\_\_\_\_\_\_\_\_\_\_\_\_\_\_\_\_\_\_\_\_
dense\_145 (Dense)            (None, 18)                288       
\_\_\_\_\_\_\_\_\_\_\_\_\_\_\_\_\_\_\_\_\_\_\_\_\_\_\_\_\_\_\_\_\_\_\_\_\_\_\_\_\_\_\_\_\_\_\_\_\_\_\_\_\_\_\_\_\_\_\_\_\_\_\_\_\_
dropout\_18 (Dropout)         (None, 18)                0         
\_\_\_\_\_\_\_\_\_\_\_\_\_\_\_\_\_\_\_\_\_\_\_\_\_\_\_\_\_\_\_\_\_\_\_\_\_\_\_\_\_\_\_\_\_\_\_\_\_\_\_\_\_\_\_\_\_\_\_\_\_\_\_\_\_
dense\_146 (Dense)            (None, 8)                 152       
\_\_\_\_\_\_\_\_\_\_\_\_\_\_\_\_\_\_\_\_\_\_\_\_\_\_\_\_\_\_\_\_\_\_\_\_\_\_\_\_\_\_\_\_\_\_\_\_\_\_\_\_\_\_\_\_\_\_\_\_\_\_\_\_\_
dropout\_19 (Dropout)         (None, 8)                 0         
\_\_\_\_\_\_\_\_\_\_\_\_\_\_\_\_\_\_\_\_\_\_\_\_\_\_\_\_\_\_\_\_\_\_\_\_\_\_\_\_\_\_\_\_\_\_\_\_\_\_\_\_\_\_\_\_\_\_\_\_\_\_\_\_\_
dense\_147 (Dense)            (None, 1)                 9         
=================================================================
Total params: 2,272
Trainable params: 2,272
Non-trainable params: 0
\_\_\_\_\_\_\_\_\_\_\_\_\_\_\_\_\_\_\_\_\_\_\_\_\_\_\_\_\_\_\_\_\_\_\_\_\_\_\_\_\_\_\_\_\_\_\_\_\_\_\_\_\_\_\_\_\_\_\_\_\_\_\_\_\_
Run Time : 14.923352

    \end{Verbatim}

\begin{Verbatim}[commandchars=\\\{\}]
{\color{outcolor}Out[{\color{outcolor}24}]:} [<matplotlib.lines.Line2D at 0x2441d9ba160>]
\end{Verbatim}

    \begin{Verbatim}[commandchars=\\\{\}]
{\color{incolor}In [{\color{incolor}25}]:} \PY{n}{model} \PY{o}{=} \PY{n}{Sequential}\PY{p}{(}\PY{p}{)}
         \PY{n}{model}\PY{o}{.}\PY{n}{add}\PY{p}{(}\PY{n}{Dense}\PY{p}{(}\PY{l+m+mi}{5}\PY{p}{,} \PY{n}{input\PYZus{}dim}\PY{o}{=}\PY{n}{input\PYZus{}dim}\PY{p}{,} \PY{n}{activation}\PY{o}{=}\PY{l+s+s1}{\PYZsq{}}\PY{l+s+s1}{relu}\PY{l+s+s1}{\PYZsq{}}\PY{p}{)}\PY{p}{)}
         \PY{n}{model}\PY{o}{.}\PY{n}{add}\PY{p}{(}\PY{n}{Dense}\PY{p}{(}\PY{l+m+mi}{10}\PY{p}{,} \PY{n}{activation}\PY{o}{=}\PY{l+s+s1}{\PYZsq{}}\PY{l+s+s1}{relu}\PY{l+s+s1}{\PYZsq{}}\PY{p}{)}\PY{p}{)}
         \PY{n}{model}\PY{o}{.}\PY{n}{add}\PY{p}{(}\PY{n}{Dense}\PY{p}{(}\PY{l+m+mi}{20}\PY{p}{,} \PY{n}{activation}\PY{o}{=}\PY{l+s+s1}{\PYZsq{}}\PY{l+s+s1}{tanh}\PY{l+s+s1}{\PYZsq{}}\PY{p}{)}\PY{p}{)}
         \PY{n}{model}\PY{o}{.}\PY{n}{add}\PY{p}{(}\PY{n}{Dense}\PY{p}{(}\PY{l+m+mi}{15}\PY{p}{,} \PY{n}{activation}\PY{o}{=}\PY{l+s+s1}{\PYZsq{}}\PY{l+s+s1}{relu}\PY{l+s+s1}{\PYZsq{}}\PY{p}{)}\PY{p}{)}
         \PY{n}{model}\PY{o}{.}\PY{n}{add}\PY{p}{(}\PY{n}{Dense}\PY{p}{(}\PY{l+m+mi}{25}\PY{p}{,} \PY{n}{activation}\PY{o}{=}\PY{l+s+s1}{\PYZsq{}}\PY{l+s+s1}{tanh}\PY{l+s+s1}{\PYZsq{}}\PY{p}{)}\PY{p}{)}
         \PY{n}{model}\PY{o}{.}\PY{n}{add}\PY{p}{(}\PY{n}{Dense}\PY{p}{(}\PY{l+m+mi}{20}\PY{p}{,} \PY{n}{activation}\PY{o}{=}\PY{l+s+s1}{\PYZsq{}}\PY{l+s+s1}{sigmoid}\PY{l+s+s1}{\PYZsq{}}\PY{p}{)}\PY{p}{)}
         \PY{n}{model}\PY{o}{.}\PY{n}{add}\PY{p}{(}\PY{n}{Dense}\PY{p}{(}\PY{l+m+mi}{25}\PY{p}{,} \PY{n}{activation}\PY{o}{=}\PY{l+s+s1}{\PYZsq{}}\PY{l+s+s1}{relu}\PY{l+s+s1}{\PYZsq{}}\PY{p}{)}\PY{p}{)}
         \PY{n}{model}\PY{o}{.}\PY{n}{add}\PY{p}{(}\PY{n}{Dense}\PY{p}{(}\PY{l+m+mi}{20}\PY{p}{,} \PY{n}{activation}\PY{o}{=}\PY{l+s+s1}{\PYZsq{}}\PY{l+s+s1}{relu}\PY{l+s+s1}{\PYZsq{}}\PY{p}{)}\PY{p}{)}
         \PY{n}{model}\PY{o}{.}\PY{n}{add}\PY{p}{(}\PY{n}{Dropout}\PY{p}{(}\PY{l+m+mf}{0.2}\PY{p}{)}\PY{p}{)}
         \PY{n}{model}\PY{o}{.}\PY{n}{add}\PY{p}{(}\PY{n}{Dense}\PY{p}{(}\PY{l+m+mi}{8}\PY{p}{,} \PY{n}{activation}\PY{o}{=}\PY{l+s+s1}{\PYZsq{}}\PY{l+s+s1}{relu}\PY{l+s+s1}{\PYZsq{}}\PY{p}{)}\PY{p}{)}
         \PY{n}{model}\PY{o}{.}\PY{n}{add}\PY{p}{(}\PY{n}{Dropout}\PY{p}{(}\PY{l+m+mf}{0.2}\PY{p}{)}\PY{p}{)}
         \PY{n}{model}\PY{o}{.}\PY{n}{add}\PY{p}{(}\PY{n}{Dense}\PY{p}{(}\PY{n}{output\PYZus{}dim}\PY{p}{,} \PY{n}{activation}\PY{o}{=}\PY{l+s+s1}{\PYZsq{}}\PY{l+s+s1}{relu}\PY{l+s+s1}{\PYZsq{}}\PY{p}{)}\PY{p}{)}
         \PY{n}{model}\PY{o}{.}\PY{n}{summary}\PY{p}{(}\PY{p}{)}
         \PY{n}{model}\PY{o}{.}\PY{n}{compile}\PY{p}{(}\PY{n}{loss}\PY{o}{=}\PY{l+s+s1}{\PYZsq{}}\PY{l+s+s1}{mean\PYZus{}squared\PYZus{}error}\PY{l+s+s1}{\PYZsq{}}\PY{p}{,} \PY{n}{optimizer}\PY{o}{=}\PY{l+s+s1}{\PYZsq{}}\PY{l+s+s1}{adam}\PY{l+s+s1}{\PYZsq{}}\PY{p}{)}
         \PY{n}{model}\PY{o}{.}\PY{n}{fit}\PY{p}{(}\PY{n}{train\PYZus{}x}\PY{p}{,}\PY{n}{train\PYZus{}y}\PY{p}{,}\PY{n}{epochs}\PY{o}{=}\PY{n}{number\PYZus{}epo}\PY{p}{,}\PY{n}{verbose}\PY{o}{=}\PY{l+m+mi}{0}\PY{p}{,}\PY{n}{batch\PYZus{}size}\PY{o}{=}\PY{l+m+mi}{10}\PY{p}{,}
                   \PY{n}{validation\PYZus{}data}\PY{o}{=}\PY{p}{(}\PY{n}{test\PYZus{}x}\PY{p}{,} \PY{n}{test\PYZus{}y}\PY{p}{)}\PY{p}{)} 
\end{Verbatim}

    \begin{Verbatim}[commandchars=\\\{\}]
\_\_\_\_\_\_\_\_\_\_\_\_\_\_\_\_\_\_\_\_\_\_\_\_\_\_\_\_\_\_\_\_\_\_\_\_\_\_\_\_\_\_\_\_\_\_\_\_\_\_\_\_\_\_\_\_\_\_\_\_\_\_\_\_\_
Layer (type)                 Output Shape              Param \#   
=================================================================
dense\_148 (Dense)            (None, 5)                 10        
\_\_\_\_\_\_\_\_\_\_\_\_\_\_\_\_\_\_\_\_\_\_\_\_\_\_\_\_\_\_\_\_\_\_\_\_\_\_\_\_\_\_\_\_\_\_\_\_\_\_\_\_\_\_\_\_\_\_\_\_\_\_\_\_\_
dense\_149 (Dense)            (None, 10)                60        
\_\_\_\_\_\_\_\_\_\_\_\_\_\_\_\_\_\_\_\_\_\_\_\_\_\_\_\_\_\_\_\_\_\_\_\_\_\_\_\_\_\_\_\_\_\_\_\_\_\_\_\_\_\_\_\_\_\_\_\_\_\_\_\_\_
dense\_150 (Dense)            (None, 20)                220       
\_\_\_\_\_\_\_\_\_\_\_\_\_\_\_\_\_\_\_\_\_\_\_\_\_\_\_\_\_\_\_\_\_\_\_\_\_\_\_\_\_\_\_\_\_\_\_\_\_\_\_\_\_\_\_\_\_\_\_\_\_\_\_\_\_
dense\_151 (Dense)            (None, 15)                315       
\_\_\_\_\_\_\_\_\_\_\_\_\_\_\_\_\_\_\_\_\_\_\_\_\_\_\_\_\_\_\_\_\_\_\_\_\_\_\_\_\_\_\_\_\_\_\_\_\_\_\_\_\_\_\_\_\_\_\_\_\_\_\_\_\_
dense\_152 (Dense)            (None, 25)                400       
\_\_\_\_\_\_\_\_\_\_\_\_\_\_\_\_\_\_\_\_\_\_\_\_\_\_\_\_\_\_\_\_\_\_\_\_\_\_\_\_\_\_\_\_\_\_\_\_\_\_\_\_\_\_\_\_\_\_\_\_\_\_\_\_\_
dense\_153 (Dense)            (None, 20)                520       
\_\_\_\_\_\_\_\_\_\_\_\_\_\_\_\_\_\_\_\_\_\_\_\_\_\_\_\_\_\_\_\_\_\_\_\_\_\_\_\_\_\_\_\_\_\_\_\_\_\_\_\_\_\_\_\_\_\_\_\_\_\_\_\_\_
dense\_154 (Dense)            (None, 25)                525       
\_\_\_\_\_\_\_\_\_\_\_\_\_\_\_\_\_\_\_\_\_\_\_\_\_\_\_\_\_\_\_\_\_\_\_\_\_\_\_\_\_\_\_\_\_\_\_\_\_\_\_\_\_\_\_\_\_\_\_\_\_\_\_\_\_
dense\_155 (Dense)            (None, 20)                520       
\_\_\_\_\_\_\_\_\_\_\_\_\_\_\_\_\_\_\_\_\_\_\_\_\_\_\_\_\_\_\_\_\_\_\_\_\_\_\_\_\_\_\_\_\_\_\_\_\_\_\_\_\_\_\_\_\_\_\_\_\_\_\_\_\_
dropout\_20 (Dropout)         (None, 20)                0         
\_\_\_\_\_\_\_\_\_\_\_\_\_\_\_\_\_\_\_\_\_\_\_\_\_\_\_\_\_\_\_\_\_\_\_\_\_\_\_\_\_\_\_\_\_\_\_\_\_\_\_\_\_\_\_\_\_\_\_\_\_\_\_\_\_
dense\_156 (Dense)            (None, 8)                 168       
\_\_\_\_\_\_\_\_\_\_\_\_\_\_\_\_\_\_\_\_\_\_\_\_\_\_\_\_\_\_\_\_\_\_\_\_\_\_\_\_\_\_\_\_\_\_\_\_\_\_\_\_\_\_\_\_\_\_\_\_\_\_\_\_\_
dropout\_21 (Dropout)         (None, 8)                 0         
\_\_\_\_\_\_\_\_\_\_\_\_\_\_\_\_\_\_\_\_\_\_\_\_\_\_\_\_\_\_\_\_\_\_\_\_\_\_\_\_\_\_\_\_\_\_\_\_\_\_\_\_\_\_\_\_\_\_\_\_\_\_\_\_\_
dense\_157 (Dense)            (None, 1)                 9         
=================================================================
Total params: 2,747
Trainable params: 2,747
Non-trainable params: 0
\_\_\_\_\_\_\_\_\_\_\_\_\_\_\_\_\_\_\_\_\_\_\_\_\_\_\_\_\_\_\_\_\_\_\_\_\_\_\_\_\_\_\_\_\_\_\_\_\_\_\_\_\_\_\_\_\_\_\_\_\_\_\_\_\_
Run Time : 16.343230

    \end{Verbatim}

\begin{Verbatim}[commandchars=\\\{\}]
{\color{outcolor}Out[{\color{outcolor}25}]:} [<matplotlib.lines.Line2D at 0x2441f61fda0>]
\end{Verbatim}

    \begin{Verbatim}[commandchars=\\\{\}]
{\color{incolor}In [{\color{incolor}26}]:} \PY{n}{model} \PY{o}{=} \PY{n}{Sequential}\PY{p}{(}\PY{p}{)}
         \PY{n}{model}\PY{o}{.}\PY{n}{add}\PY{p}{(}\PY{n}{Dense}\PY{p}{(}\PY{l+m+mi}{40}\PY{p}{,} \PY{n}{input\PYZus{}dim}\PY{o}{=}\PY{n}{input\PYZus{}dim}\PY{p}{,} \PY{n}{activation}\PY{o}{=}\PY{l+s+s1}{\PYZsq{}}\PY{l+s+s1}{relu}\PY{l+s+s1}{\PYZsq{}}\PY{p}{)}\PY{p}{)}
         \PY{n}{model}\PY{o}{.}\PY{n}{add}\PY{p}{(}\PY{n}{Dense}\PY{p}{(}\PY{l+m+mi}{25}\PY{p}{,} \PY{n}{activation}\PY{o}{=}\PY{l+s+s1}{\PYZsq{}}\PY{l+s+s1}{relu}\PY{l+s+s1}{\PYZsq{}}\PY{p}{)}\PY{p}{)}
         \PY{n}{model}\PY{o}{.}\PY{n}{add}\PY{p}{(}\PY{n}{Dropout}\PY{p}{(}\PY{l+m+mf}{0.2}\PY{p}{)}\PY{p}{)}
         \PY{n}{model}\PY{o}{.}\PY{n}{add}\PY{p}{(}\PY{n}{Dense}\PY{p}{(}\PY{n}{output\PYZus{}dim}\PY{p}{,} \PY{n}{activation}\PY{o}{=}\PY{l+s+s1}{\PYZsq{}}\PY{l+s+s1}{relu}\PY{l+s+s1}{\PYZsq{}}\PY{p}{)}\PY{p}{)}
         \PY{n}{model}\PY{o}{.}\PY{n}{add}\PY{p}{(}\PY{n}{Dropout}\PY{p}{(}\PY{l+m+mf}{0.2}\PY{p}{)}\PY{p}{)}
         \PY{n}{model}\PY{o}{.}\PY{n}{summary}\PY{p}{(}\PY{p}{)}
         \PY{n}{model}\PY{o}{.}\PY{n}{compile}\PY{p}{(}\PY{n}{loss}\PY{o}{=}\PY{l+s+s1}{\PYZsq{}}\PY{l+s+s1}{mean\PYZus{}squared\PYZus{}error}\PY{l+s+s1}{\PYZsq{}}\PY{p}{,} \PY{n}{optimizer}\PY{o}{=}\PY{l+s+s1}{\PYZsq{}}\PY{l+s+s1}{adam}\PY{l+s+s1}{\PYZsq{}}\PY{p}{)}
         \PY{n}{model}\PY{o}{.}\PY{n}{fit}\PY{p}{(}\PY{n}{train\PYZus{}x}\PY{p}{,}\PY{n}{train\PYZus{}y}\PY{p}{,}\PY{n}{epochs}\PY{o}{=}\PY{n}{number\PYZus{}epo}\PY{p}{,}\PY{n}{verbose}\PY{o}{=}\PY{l+m+mi}{0}\PY{p}{,}\PY{n}{batch\PYZus{}size}\PY{o}{=}\PY{l+m+mi}{10}\PY{p}{,}
                   \PY{n}{validation\PYZus{}data}\PY{o}{=}\PY{p}{(}\PY{n}{test\PYZus{}x}\PY{p}{,} \PY{n}{test\PYZus{}y}\PY{p}{)}\PY{p}{)} 
\end{Verbatim}

    \begin{Verbatim}[commandchars=\\\{\}]
\_\_\_\_\_\_\_\_\_\_\_\_\_\_\_\_\_\_\_\_\_\_\_\_\_\_\_\_\_\_\_\_\_\_\_\_\_\_\_\_\_\_\_\_\_\_\_\_\_\_\_\_\_\_\_\_\_\_\_\_\_\_\_\_\_
Layer (type)                 Output Shape              Param \#   
=================================================================
dense\_158 (Dense)            (None, 40)                80        
\_\_\_\_\_\_\_\_\_\_\_\_\_\_\_\_\_\_\_\_\_\_\_\_\_\_\_\_\_\_\_\_\_\_\_\_\_\_\_\_\_\_\_\_\_\_\_\_\_\_\_\_\_\_\_\_\_\_\_\_\_\_\_\_\_
dense\_159 (Dense)            (None, 25)                1025      
\_\_\_\_\_\_\_\_\_\_\_\_\_\_\_\_\_\_\_\_\_\_\_\_\_\_\_\_\_\_\_\_\_\_\_\_\_\_\_\_\_\_\_\_\_\_\_\_\_\_\_\_\_\_\_\_\_\_\_\_\_\_\_\_\_
dropout\_22 (Dropout)         (None, 25)                0         
\_\_\_\_\_\_\_\_\_\_\_\_\_\_\_\_\_\_\_\_\_\_\_\_\_\_\_\_\_\_\_\_\_\_\_\_\_\_\_\_\_\_\_\_\_\_\_\_\_\_\_\_\_\_\_\_\_\_\_\_\_\_\_\_\_
dense\_160 (Dense)            (None, 1)                 26        
\_\_\_\_\_\_\_\_\_\_\_\_\_\_\_\_\_\_\_\_\_\_\_\_\_\_\_\_\_\_\_\_\_\_\_\_\_\_\_\_\_\_\_\_\_\_\_\_\_\_\_\_\_\_\_\_\_\_\_\_\_\_\_\_\_
dropout\_23 (Dropout)         (None, 1)                 0         
=================================================================
Total params: 1,131
Trainable params: 1,131
Non-trainable params: 0
\_\_\_\_\_\_\_\_\_\_\_\_\_\_\_\_\_\_\_\_\_\_\_\_\_\_\_\_\_\_\_\_\_\_\_\_\_\_\_\_\_\_\_\_\_\_\_\_\_\_\_\_\_\_\_\_\_\_\_\_\_\_\_\_\_
Run Time : 14.620049

    \end{Verbatim}

\begin{Verbatim}[commandchars=\\\{\}]
{\color{outcolor}Out[{\color{outcolor}26}]:} [<matplotlib.lines.Line2D at 0x24421b63da0>]
\end{Verbatim}

    \begin{Verbatim}[commandchars=\\\{\}]
{\color{incolor}In [{\color{incolor}27}]:} \PY{n}{model} \PY{o}{=} \PY{n}{Sequential}\PY{p}{(}\PY{p}{)}
         \PY{n}{model}\PY{o}{.}\PY{n}{add}\PY{p}{(}\PY{n}{Dense}\PY{p}{(}\PY{l+m+mi}{5}\PY{p}{,} \PY{n}{input\PYZus{}dim}\PY{o}{=}\PY{n}{input\PYZus{}dim}\PY{p}{,} \PY{n}{activation}\PY{o}{=}\PY{l+s+s1}{\PYZsq{}}\PY{l+s+s1}{soft\PYZus{}plus\PYZus{}te}\PY{l+s+s1}{\PYZsq{}}\PY{p}{)}\PY{p}{)}
         \PY{n}{model}\PY{o}{.}\PY{n}{add}\PY{p}{(}\PY{n}{Dense}\PY{p}{(}\PY{l+m+mi}{10}\PY{p}{,} \PY{n}{activation}\PY{o}{=}\PY{l+s+s1}{\PYZsq{}}\PY{l+s+s1}{soft\PYZus{}plus\PYZus{}te}\PY{l+s+s1}{\PYZsq{}}\PY{p}{)}\PY{p}{)}
         \PY{n}{model}\PY{o}{.}\PY{n}{add}\PY{p}{(}\PY{n}{Dense}\PY{p}{(}\PY{l+m+mi}{20}\PY{p}{,} \PY{n}{activation}\PY{o}{=}\PY{l+s+s1}{\PYZsq{}}\PY{l+s+s1}{tanh}\PY{l+s+s1}{\PYZsq{}}\PY{p}{)}\PY{p}{)}
         \PY{n}{model}\PY{o}{.}\PY{n}{add}\PY{p}{(}\PY{n}{Dense}\PY{p}{(}\PY{l+m+mi}{15}\PY{p}{,} \PY{n}{activation}\PY{o}{=}\PY{l+s+s1}{\PYZsq{}}\PY{l+s+s1}{relu}\PY{l+s+s1}{\PYZsq{}}\PY{p}{)}\PY{p}{)}
         \PY{n}{model}\PY{o}{.}\PY{n}{add}\PY{p}{(}\PY{n}{Dense}\PY{p}{(}\PY{l+m+mi}{25}\PY{p}{,} \PY{n}{activation}\PY{o}{=}\PY{l+s+s1}{\PYZsq{}}\PY{l+s+s1}{tanh}\PY{l+s+s1}{\PYZsq{}}\PY{p}{)}\PY{p}{)}
         \PY{n}{model}\PY{o}{.}\PY{n}{add}\PY{p}{(}\PY{n}{Dense}\PY{p}{(}\PY{l+m+mi}{20}\PY{p}{,} \PY{n}{activation}\PY{o}{=}\PY{l+s+s1}{\PYZsq{}}\PY{l+s+s1}{sigmoid}\PY{l+s+s1}{\PYZsq{}}\PY{p}{)}\PY{p}{)}
         \PY{n}{model}\PY{o}{.}\PY{n}{add}\PY{p}{(}\PY{n}{Dense}\PY{p}{(}\PY{l+m+mi}{25}\PY{p}{,} \PY{n}{activation}\PY{o}{=}\PY{l+s+s1}{\PYZsq{}}\PY{l+s+s1}{relu}\PY{l+s+s1}{\PYZsq{}}\PY{p}{)}\PY{p}{)}
         \PY{n}{model}\PY{o}{.}\PY{n}{add}\PY{p}{(}\PY{n}{Dense}\PY{p}{(}\PY{n}{output\PYZus{}dim}\PY{p}{,} \PY{n}{activation}\PY{o}{=}\PY{l+s+s1}{\PYZsq{}}\PY{l+s+s1}{soft\PYZus{}plus\PYZus{}te}\PY{l+s+s1}{\PYZsq{}}\PY{p}{)}\PY{p}{)}
         \PY{n}{model}\PY{o}{.}\PY{n}{add}\PY{p}{(}\PY{n}{Dropout}\PY{p}{(}\PY{l+m+mf}{0.2}\PY{p}{)}\PY{p}{)}
         \PY{n}{model}\PY{o}{.}\PY{n}{compile}\PY{p}{(}\PY{n}{loss}\PY{o}{=}\PY{l+s+s1}{\PYZsq{}}\PY{l+s+s1}{mean\PYZus{}squared\PYZus{}error}\PY{l+s+s1}{\PYZsq{}}\PY{p}{,} \PY{n}{optimizer}\PY{o}{=}\PY{l+s+s1}{\PYZsq{}}\PY{l+s+s1}{adam}\PY{l+s+s1}{\PYZsq{}}\PY{p}{)}
         \PY{n}{model}\PY{o}{.}\PY{n}{fit}\PY{p}{(}\PY{n}{train\PYZus{}x}\PY{p}{,}\PY{n}{train\PYZus{}y}\PY{p}{,}\PY{n}{epochs}\PY{o}{=}\PY{n}{number\PYZus{}epo}\PY{p}{,}\PY{n}{verbose}\PY{o}{=}\PY{l+m+mi}{0}\PY{p}{,}\PY{n}{batch\PYZus{}size}\PY{o}{=}\PY{l+m+mi}{10}\PY{p}{,}
                   \PY{n}{validation\PYZus{}data}\PY{o}{=}\PY{p}{(}\PY{n}{test\PYZus{}x}\PY{p}{,} \PY{n}{test\PYZus{}y}\PY{p}{)}\PY{p}{)}
         \PY{n}{model}\PY{o}{.}\PY{n}{summary}\PY{p}{(}\PY{p}{)}
\end{Verbatim}

    \begin{Verbatim}[commandchars=\\\{\}]
\_\_\_\_\_\_\_\_\_\_\_\_\_\_\_\_\_\_\_\_\_\_\_\_\_\_\_\_\_\_\_\_\_\_\_\_\_\_\_\_\_\_\_\_\_\_\_\_\_\_\_\_\_\_\_\_\_\_\_\_\_\_\_\_\_
Layer (type)                 Output Shape              Param \#   
=================================================================
dense\_161 (Dense)            (None, 5)                 10        
\_\_\_\_\_\_\_\_\_\_\_\_\_\_\_\_\_\_\_\_\_\_\_\_\_\_\_\_\_\_\_\_\_\_\_\_\_\_\_\_\_\_\_\_\_\_\_\_\_\_\_\_\_\_\_\_\_\_\_\_\_\_\_\_\_
dense\_162 (Dense)            (None, 10)                60        
\_\_\_\_\_\_\_\_\_\_\_\_\_\_\_\_\_\_\_\_\_\_\_\_\_\_\_\_\_\_\_\_\_\_\_\_\_\_\_\_\_\_\_\_\_\_\_\_\_\_\_\_\_\_\_\_\_\_\_\_\_\_\_\_\_
dense\_163 (Dense)            (None, 20)                220       
\_\_\_\_\_\_\_\_\_\_\_\_\_\_\_\_\_\_\_\_\_\_\_\_\_\_\_\_\_\_\_\_\_\_\_\_\_\_\_\_\_\_\_\_\_\_\_\_\_\_\_\_\_\_\_\_\_\_\_\_\_\_\_\_\_
dense\_164 (Dense)            (None, 15)                315       
\_\_\_\_\_\_\_\_\_\_\_\_\_\_\_\_\_\_\_\_\_\_\_\_\_\_\_\_\_\_\_\_\_\_\_\_\_\_\_\_\_\_\_\_\_\_\_\_\_\_\_\_\_\_\_\_\_\_\_\_\_\_\_\_\_
dense\_165 (Dense)            (None, 25)                400       
\_\_\_\_\_\_\_\_\_\_\_\_\_\_\_\_\_\_\_\_\_\_\_\_\_\_\_\_\_\_\_\_\_\_\_\_\_\_\_\_\_\_\_\_\_\_\_\_\_\_\_\_\_\_\_\_\_\_\_\_\_\_\_\_\_
dense\_166 (Dense)            (None, 20)                520       
\_\_\_\_\_\_\_\_\_\_\_\_\_\_\_\_\_\_\_\_\_\_\_\_\_\_\_\_\_\_\_\_\_\_\_\_\_\_\_\_\_\_\_\_\_\_\_\_\_\_\_\_\_\_\_\_\_\_\_\_\_\_\_\_\_
dense\_167 (Dense)            (None, 25)                525       
\_\_\_\_\_\_\_\_\_\_\_\_\_\_\_\_\_\_\_\_\_\_\_\_\_\_\_\_\_\_\_\_\_\_\_\_\_\_\_\_\_\_\_\_\_\_\_\_\_\_\_\_\_\_\_\_\_\_\_\_\_\_\_\_\_
dense\_168 (Dense)            (None, 1)                 26        
\_\_\_\_\_\_\_\_\_\_\_\_\_\_\_\_\_\_\_\_\_\_\_\_\_\_\_\_\_\_\_\_\_\_\_\_\_\_\_\_\_\_\_\_\_\_\_\_\_\_\_\_\_\_\_\_\_\_\_\_\_\_\_\_\_
dropout\_24 (Dropout)         (None, 1)                 0         
=================================================================
Total params: 2,076
Trainable params: 2,076
Non-trainable params: 0
\_\_\_\_\_\_\_\_\_\_\_\_\_\_\_\_\_\_\_\_\_\_\_\_\_\_\_\_\_\_\_\_\_\_\_\_\_\_\_\_\_\_\_\_\_\_\_\_\_\_\_\_\_\_\_\_\_\_\_\_\_\_\_\_\_

    \end{Verbatim}

\end{document}